\documentclass[11pt]{article}

\usepackage[margin=1in]{geometry}
\usepackage{mathtools}
\mathtoolsset{showonlyrefs=true}
\usepackage{amssymb,amsfonts,amsthm}
\usepackage{xcolor}
\usepackage{graphicx}
\usepackage{epstopdf}
\usepackage{algorithmic}
\usepackage{booktabs}
\usepackage{multirow}
\usepackage{tabularx}
\usepackage{array}
\usepackage{tikz}
\usepackage{subcaption}
\usepackage{float}
\usepackage[hidelinks]{hyperref}
\usepackage[nameinlink,capitalize,noabbrev]{cleveref}

\usetikzlibrary{positioning, calc, decorations.pathreplacing}
\DeclareGraphicsExtensions{.pdf,.png,.jpg,.eps}

\newcolumntype{Y}{>{\centering\arraybackslash}X}

\numberwithin{equation}{section}

\newtheorem{theorem}{Theorem}[section]
\newtheorem{lemma}{Lemma}[section]

\theoremstyle{definition}
\newtheorem{definition}{Definition}[section]

\theoremstyle{remark}

\crefname{hypothesis}{Hypothesis}{Hypotheses}
\crefname{fact}{Fact}{Facts}
\newcommand{\E}{\sigma^{+}_{\theta}}
\newcommand{\D}{\sigma_{\theta}}

\hypersetup{
  pdftitle={Learning symplectic model reduction based on an approximation theorem of symplectic embeddings},
  pdfauthor={Liyi Feng, Yifa Tang, Yulin Xie, Ruili Zhang, and Aiqing Zhu}
}

\title{Learning symplectic model reduction based on an approximation theorem of symplectic embeddings}
\author{\begin{minipage}{0.95\textwidth}
\centering
Liyi Feng$^{1}$, Yifa Tang$^{2}$, Yulin Xie$^{2}$, Ruili Zhang$^{1,*}$, and Aiqing Zhu$^{3,*}$\\[0.5em]
\small $^{1}$School of Mathematics and Statistics, Beijing Jiaotong University, Beijing 100044, China\\
\small $^{2}$State Key Laboratory of Mathematical Sciences, Academy of Mathematics and Systems Science, Chinese Academy of Sciences, Beijing 100190, China\\
\small $^{3}$Department of Mathematics, National University of Singapore, Singapore 119076\\
\small $^{*}$Corresponding authors: \texttt{zhangrl@bjtu.edu.cn}, \texttt{zaq@nus.edu.sg}
\end{minipage}}
\date{}

\begin{document}

\maketitle

\begin{abstract}
High–dimensional Hamiltonian systems play a central role in many scientific and engineering disciplines, with dynamics that evolve on symplectic manifolds. Although deep learning provides powerful tools for constructing its low-dimensional surrogates from data, the intrinsic symplectic structure is easily destroyed during model reduction. As a result, a standard autoencoder may produce latent coordinates that do not support a Hamiltonian flow, leading to unstable long-time prediction. In this paper, we first establish a universal approximation theorem for symplectic embeddings. And based on the theory, we propose symplecticity-preserving autoencoders (SpAE), in which the decoder is parameterized as a symplectic embedding and the encoder is constructed as the corresponding symplectic projection. This architecture is expressive enough to approximate nonlinear symplectic embeddings and the corresponding symplectic projection, preserves the symplectic structure exactly by construction, and can be trained by standard unconstrained optimization, thereby improving both reconstruction and prediction accuracy. Extensive experiments on high–dimensional lattice and particle systems demonstrate the effectiveness of the proposed method.
\end{abstract}

\begin{center}
\begin{minipage}{0.85\textwidth}

\end{minipage}
\end{center}

\section{Introduction}
High–dimensional data frequently arise in scientific computing, physical simulations, and modern artificial intelligence. A central objective is to obtain low–dimensional representations that preserve essential structure while enabling efficient learning and prediction~\cite{maaten2008visualizing,mcinnes2018umap}. Recent progress in deep representation learning has led to powerful nonlinear dimensionality-reduction techniques based on autoencoders~\cite{hinton2006reducing,vincent2008extracting,wang2016auto}, contrastive learning~\cite{abid2019contrastive,hadsell2006dimensionality,hawke2024contrastive}, and diffusion models~\cite{preechakul2022diffusion,wu2024factorized}. These approaches have demonstrated remarkable success in capturing latent manifolds of images, text, and other unstructured data. However, when the data originate from dynamical systems, particularly those governed by strict physical laws, standard dimensionality-reduction methods face intrinsic limitations: they typically optimize reconstruction accuracy or statistical fidelity, but pay limited attention to the underlying geometric or physical structure that drives the evolution of the data. As a consequence, the learned representations may fail to support long-horizon prediction or physically consistent modeling~\cite{brantner2023symplectic,peng2016symplectic,  zhu2025continuity}.

Hamiltonian systems provide a prominent and challenging class of examples. Their trajectories evolve on symplectic manifolds, satisfying exact conservation of volume and energy along the flow~\cite{Arnold1989}. For such systems, the preservation of symplectic structure is essential for qualitative fidelity~\cite{feng1985proceedings,feng1986difference,feng1995collected,feng2010symplectic}. In particular, if a dimensionality-reduction map does not preserve the symplectic form, then the reduced variables cannot correspond to any Hamiltonian flow, and the predictive model built in the reduced space may exhibit drift, instability, or nonphysical behavior when lifted back to the original space~\cite{hairer2006geometric,reich1999backward}. Existing linear symplectic reduction methods, often collectively formulated as proper symplectic decomposition (PSD), are effective in moderate dimensions, and have been successfully employed in several model-reduction frameworks~\cite{bajars2025structure,brantner2023symplectic,buchfink2023symplectic,peng2016symplectic,sharma2023symplectic,yildiz2025symplectic}.

However, these inherently linear methods can become inefficient when the solution manifold is intrinsically nonlinear, often requiring a large number of modes to achieve acceptable accuracy; see, e.g.,~\cite{blickhan2024registration,brantner2023symplectic,lee2020model} and also our numerical results below. This limitation naturally motivates the use of nonlinear reduced representations. In particular, autoencoder-based approaches provide a flexible framework for learning nonlinear trial manifolds directly from data~\cite{chen2024constructing, zhu2025continuity, zhu2025identifiable}. For Hamiltonian systems, the key additional requirement is that the reduction preserve the underlying symplectic structure. Existing works have begun to address this issue from different perspectives. Buchfink et al.~\cite{buchfink2023symplectic} proposed a nonlinear symplectic model reduction framework based on symplectic manifold Galerkin projection, using a symplecticity penalty to encourage the decoder to define an approximately symplectic trial manifold. More recently, Brantner and Kraus~\cite{brantner2023symplectic} introduced a symplectic autoencoder architecture that enforces symplecticity by construction through a combination of nonlinear symplectic network blocks and PSD-like linear layers. Nevertheless, the latter approach still relies on manifold-constrained optimization in training.

In this work, we first develop an approximation theory for general symplectic embeddings.
While existing approximation results have treated specific symplectic embeddings \cite{chen2025reduced}, our analysis covers the general case.
Based on our theory, we then construct symplecticity-preserving autoencoders (SpAE) that learn a Hamiltonian reduction map from the original phase space to a low-dimensional latent space while preserving the symplecticity form intrinsically.
Unlike approaches that promote symplecticity through post hoc projection or penalty terms, our architecture guarantees exact symplecticity by construction. Consequently, the proposed SpAE is fully compatible with standard autoencoder training pipelines.
Several numerical experiments demonstrate the effectiveness of this approach. Across a range of high-dimensional Hamiltonian systems, the proposed SpAE yields consistently improved reconstruction accuracy. Moreover, by learning the latent dynamics with Hamiltonian neural networks in the reduced space, the resulting reduced models achieve substantially improved long-time trajectory prediction.

The remainder of this paper is organized as follows. \Cref{sec:embeddings} defines symplectic embeddings and projections, shows that they reduce the dynamics to canonical Hamiltonian form, and then establishes the approximation theory. \Cref{sec:architecture} uses this theory to build the symplecticity-preserving encoder, decoder, and autoencoder, together with the training objective, and \Cref{sec:proofs} provides the proofs of the approximation results. \Cref{sec:experiments} validates SpAE on a crystal-lattice model, charged particles in a tokamak field, and a two-stream plasma instability, and \Cref{sec:conclusion} concludes the paper.

\section{Approximation of symplectic embeddings}\label{sec:embeddings}
Given snapshot data generated by a Hamiltonian system on the phase space $\mathbb{R}^{2n}$, our objective is to learn a nonlinear low-dimensional representation of its solution manifold in a latent phase space $\mathbb{R}^{2k}$, where $k\ll n$, and to learn the corresponding latent dynamics.
To this end, we seek an encoder--decoder pair
\[
\E:\mathbb{R}^{2n}\to\mathbb{R}^{2k},
\qquad
\D:\mathbb{R}^{2k}\to\mathbb{R}^{2n},
\]
such that $\D\circ \E$ accurately reconstructs the solution manifold while, at the same time, the decoder $\D$ defines a symplectic embedding from the latent phase space into the original phase space and the encoder $\E$ defines the corresponding symplectic projection. In this section, we first show that this requirement ensures that the reduced dynamics in the latent space admits a Hamiltonian formulation. We then establish an approximation theorem for symplectic embeddings.

\subsection{Hamiltonian system, Symplectic embedding and projection}
We first present some of the notations and definitions
introduced throughout this paper.
\begin{definition}[Symplectic matrix]
A matrix $B \in \mathbb{R}^{2d \times 2d}$ is called  symplectic if it satisfies $B^\top J_{2d} B = J_{2d}$. The set of all such matrices, denoted by
\begin{equation*}
\mathrm{Sp}(2d, \mathbb{R}) = \{ B \in \mathbb{R}^{2d \times 2d} \mid B^\top J_{2d} B = J_{2d} \},
\end{equation*}
is called the {matrix symplectic group}.
\end{definition}
\begin{definition}[Symplectic map]
For an open set $U \subset \mathbb{R}^{2d}$, a differentiable map $\Phi: U \rightarrow \mathbb{R}^{2d}$ is called symplectic if its Jacobian matrix is symplectic for every $y\in U$. Namely, \begin{equation}
\left(D \Phi(y) \right)^\top J_{2d}
\left(D \Phi(y)\right)=J_{2d},\quad \forall y \in U,
\end{equation}
where $D \Phi(y)\in\mathbb{R}^{2d\times 2d}$ denotes the Jacobian matrix of $\Phi$ at $y$.
\end{definition}

Consider a Hamiltonian system on the phase space $\mathbb{R}^{2d}$,
\begin{equation}\label{eq:hamiltonian}
    \dot{y}(t)=J_{2d}\nabla H(y(t)),\qquad  J_{2d}=
\begin{pmatrix}
0 & I_d\\
-I_d & 0
\end{pmatrix},
\end{equation}
where $H:\mathbb{R}^{2d}\to\mathbb{R}$ is the Hamiltonian and $I_d$ is the $d\times d$ identity matrix.
A classical result is that the phase flow of a Hamiltonian system is a symplectic map; see, for example,~\cite[Theorem~2.4]{hairer2006geometric}.
This symplectic structure is closely tied to the long-time stability of Hamiltonian dynamics and is therefore highly valued in numerical computation \cite{channell1990symplectic,feng1985proceedings,feng1986difference,feng1995collected,feng2010symplectic,forest1990fourth,hairer2006geometric,sanz1988runge,sanz2018numerical}, which in turn motivates our focus on symplectic reduction.

\begin{definition}[Symplectic embedding~\cite{cristofaro2018symplectic}]
\label{def:symplectic_embedding}
A smooth embedding $\sigma:\mathbb{R}^{2k}\hookrightarrow\mathbb{R}^{2n}$ (with $k < n$) is
a \emph{symplectic embedding} if its Jacobian satisfies
\begin{equation}
    (D\sigma(p))^\top J_{2n}\,D\sigma(p)=J_{2k},\quad \text{for all }p\in\mathbb{R}^{2k},
    \label{eq:symp-embed}
\end{equation}
where $D\sigma(p)\in\mathbb{R}^{2n\times 2k}$ denotes the Jacobian of $\sigma$ at $p$.
\end{definition}

Then we give the general definition of the symplectic projection as follows.
\begin{definition}[Symplectic projection]
\label{def:sym_projection}
Let $\sigma:\mathbb{R}^{2k}\to\mathbb{R}^{2n}$ be a symplectic embedding, and let
$\mathcal M:=\sigma(\mathbb{R}^{2k})$.
A smooth map $\sigma^+:U\to\mathbb{R}^{2k}$, defined on a neighborhood $U$ of $\mathcal M$,
is called an \emph{associated symplectic projection} of $\sigma$ if
\begin{equation*}
    \sigma^+\circ \sigma = \mathrm{Id}_{\mathbb{R}^{2k}},\quad \text{and} \quad
    J_{2k} D\sigma^+(\sigma(p))=\,(D\sigma(p))^\top J_{2n} ,
    \qquad \forall p\in\mathbb{R}^{2k}.
\end{equation*}
\end{definition}
It is easy to verify that, in the linear setting, our definition is equivalent to that in~\cite{peng2016structure}.

We now explain how symplectic embeddings and their projections interact with Hamiltonian reduction. Let \(\sigma:\mathbb{R}^{2k}\to\mathbb{R}^{2n}\) be a symplectic embedding with image \(\mathcal M=\operatorname{Im}(\sigma)\subset\mathbb{R}^{2n}\), a \(2k\)-dimensional symplectic submanifold parameterized by reduced coordinates \(z\in\mathbb{R}^{2k}\). When the Hamiltonian dynamics in \(\mathbb{R}^{2n}\) are restricted to \(\mathcal M\), the reduced coordinates \(z\) satisfy a reduced Hamiltonian system, as shown in the following theorem.

\begin{lemma}
Consider the canonical Hamiltonian system $\dot{x} = J_{2n}\nabla H(x)$, where $x \in \mathbb{R}^{2n}$, and
$H: \mathbb{R}^{2n} \to \mathbb{R}$ is the Hamiltonian function. Let $\sigma: \mathbb{R}^{2k} \to \mathbb{R}^{2n}$ be a smooth symplectic embedding. Suppose that $x(t)$ is a solution and $x(t) \in \operatorname{Im}(\sigma)$ for all $t$ in the time interval $I$. Then there exists a curve $z : I \to \mathbb{R}^{2k}$ such that $x(t) = \sigma(z(t))$, and $z(t)$ satisfies
\begin{equation*}
  \dot{z} = J_{2k}\,\nabla H_r(z),\quad  H_r(z) := H(\sigma(z)),
  \quad z \in \mathbb{R}^{2k}.
\end{equation*}
\end{lemma}

\begin{proof}
Since $x(t) \in \operatorname{Im}(\sigma)$ and $\sigma$ is a symplectic embedding, we write $x(t) = \sigma(z(t))$. Substituting the time derivative $\dot{x} = D\sigma(z)\dot{z}$ into the original system $\dot{x} = J_{2n}\nabla H(x)$ yields
\begin{equation*}
        D\sigma(z)\,\dot{z} = J_{2n}\,\nabla H(\sigma(z)).
\end{equation*}
Since $(D\sigma(p))^\top J_{2n}\,D\sigma(p)=J_{2k}$ and the chain rule $\nabla H_r(z) = D\sigma(z)^\top \nabla H(\sigma(z))$, the expression simplifies to $J_{2k}\dot{z} = -\nabla H_r(z)$. Left-multiplying by $J_{2k}^{-1} = -J_{2k}$ yields the canonical reduced Hamiltonian system.
\end{proof}

When the dynamics do not lie exactly on \(\mathcal M\), the following lemma shows that the reduced dynamics still approximately preserves the Hamiltonian structure, provided that \(\sigma\) is a symplectic embedding and \(\sigma^{+}\) is its associated symplectic projection.

\begin{lemma}
\label{lemma:Ham qpproach}
Let $\Omega\subset \mathbb{R}^{2n}$ be compact and
$H\in C^2(\mathbb{R}^{2n})$. Let $\sigma$ be a symplectic embedding and $\sigma^{+}$ be its associated symplectic projection. Both $\sigma$ and $\sigma^{+}$ are assumed to be smooth.
Assume that
\[
\sup_{x\in\Omega}|x-\sigma\circ \sigma^{+}(x)|\le \varepsilon .
\]
Let $x(t):[0,T]\to \Omega$ solve
\(
\dot x=J_{2n}\nabla H(x),
\)
and let $z(t):[0,T]\to \mathbb{R}^{2k}$ solve
\(
\dot z=J_{2k}\nabla(H\circ \sigma)(z),\ z(0)=\sigma^{+}(x(0)).
\)
Then there exists a constant $C_T>0$, independent of $\varepsilon$, such that
\[
\sup_{t\in[0,T]}|z(t)-\sigma^{+}(x(t))|\le C_T\varepsilon .
\]
\end{lemma}
\begin{proof}
Set
\[
y(t):=\sigma^{+}(x(t)),\qquad F(z): = J_{2k}\nabla(H\circ \sigma)(z).
\]
By the chain rule,
\(
F(z)=J_{2k}D\sigma(z)^T\nabla H(\sigma(z)).
\)
Since $\sigma^{+}$ is the symplectic projection associated with $\sigma$, we have
\[
F(z)=D\sigma^{+}(\sigma(z))\,J_{2n}\nabla H(\sigma(z)).
\]
On the other hand, we have
\(
\dot y(t)=D\sigma^{+}(x(t))\,J_{2n}\nabla H(x(t)).
\)
Therefore,
\[
\dot y(t)-F(y(t))=r(t),\quad |r(t)|\le C\varepsilon,
\]
where we have used the identity $\sigma(y(t))=\sigma\circ\sigma^{+}(x(t))$, the bound
\(
|x(t)-\sigma(y(t))|\le \varepsilon,
\)
and the regularity of $H$ and $\sigma^{+}$.
Finally the conclusion follows from a standard Gronwall's argument.
\end{proof}

\subsection{Approximation Theorem for Symplectic embeddings}\label{sec:approximation}
Since the decoder is required to be a symplectic embedding, a central question is whether such maps can be approximated by neural network class. We therefore turn to the approximation of symplectic embeddings. In this section, we present the main results, and the proofs are deferred to \Cref{sec:proofs}.

\begin{theorem}[Factorization of linear symplectic embeddings]
\label{thm:Factorization of linear}
Let $A \in \mathbb{R}^{2n \times 2k}$ be a linear symplectic embedding satisfying $A^\top J_{2n} A = J_{2k}$. Then there exist symmetric matrices $S_i\in\mathbb{R}^{n \times n}$ such that $A$ can be factored as:
\begin{equation}
    A = \begin{pmatrix}I_n & S_5\\0 & I_n\end{pmatrix} \begin{pmatrix}I_n & 0\\S_4 & I_n\end{pmatrix}\begin{pmatrix}I_n & S_3\\0 & I_n\end{pmatrix} \begin{pmatrix}I_n & 0\\S_2 & I_n\end{pmatrix}\begin{pmatrix}I_n & S_1\\0 & I_n\end{pmatrix} S_c,
\end{equation}
where $S_c$ is the canonical linear symplectic embedding defined as
\begin{equation}\label{eq:sc}
S_c = \begin{pmatrix}
I_k & 0_{k \times k} \\
0_{(n-k) \times k} & 0_{(n-k) \times k} \\
0_{k \times k} & I_k \\
0_{(n-k) \times k} & 0_{(n-k) \times k}
\end{pmatrix} \in \mathbb{R}^{2n \times 2k}.
\end{equation}
\end{theorem}

\Cref{thm:Factorization of linear} reveals that every linear symplectic embedding decomposes into a product of unit triangular symplectic matrices acting on the canonical linear embedding $S_c$. This triangular structure motivates a natural nonlinear generalization: replacing symmetric matrices $S_i$ by gradient maps $\nabla V_i$ of scalar potentials yields elementary symplectic shear layers that remain exactly symplectic. We now show that composing finitely many such nonlinear layers is sufficient to approximate any smooth symplectic embedding to arbitrary accuracy.

\begin{theorem}[Approximation of nonlinear symplectic embeddings]
\label{thm:Approximation of nonlinear}
Let $\Omega\subset\mathbb R^{2k}$ be open, let
$\mathcal D\subset\Omega$ be compact, and let
$\sigma:\Omega\to\mathbb R^{2n}$ be a smooth symplectic embedding. Assume that
\(\mathcal D\) is contained in a contractible open subset of \(\Omega\). For any
$\epsilon > 0$, there exist a positive integer $L$ and polynomial potentials
$V_1,\dots,V_L:\mathbb{R}^n\to\mathbb{R}$ such that
\begin{equation}
    \left\| \sigma - \mathcal S_L \circ \cdots \circ \mathcal S_2 \circ \mathcal S_1 \circ \mathcal{S_C} \right\|_{C^0(\mathcal{D})} < \epsilon ,
\end{equation}
where $\mathcal{S_C}(z) = S_c z$ is the canonical linear symplectic embedding and the shear maps alternate according to the parity of the index,
\[
\mathcal S_{2m-1}(q,p)
=
\bigl(q+\nabla V_{2m-1}(p),\, p\bigr), \quad
\mathcal S_{2m}(q,p)
=
\bigl(q,\, p+\nabla V_{2m}(q)\bigr).
\]

\end{theorem}

\Cref{thm:Approximation of nonlinear} reduces the approximation of an arbitrary symplectic embedding to the approximation of the scalar potential functions $V_i$. Consequently, any class of function approximators that is dense in $C^1$ on compact sets, such as fully-connected networks, residual networks, or convolutional networks, immediately yields a universally approximating family of symplectic embeddings. According to \Cref{thm:Approximation of nonlinear}, we can directly parameterize each $V_i$ by a neural network, and the universal approximation property of the resulting symplectic decoder is an immediate corollary of the theorem above, as presented in the next section.

We note that concurrent work established a related approximation result based on HénonNet composite embeddings \cite{chen2025reduced}. Their result does not require a contractibility assumption because it is developed for symplectic embeddings with the prescribed structure
\[
\begin{cases}
\hat{x} = Py + \eta,\\
\hat{y} = -Px + \nabla V(Py).
\end{cases}
\]
By contrast, we consider general symplectic embeddings in the standard sense \cite{cristofaro2018symplectic}, without imposing such a structural form. This generality requires the compact set under consideration to be contained in a contractible open subset, thereby ensuring the trivialization of the associated symplectic normal bundle required in our analysis.

\section{Architecture of symplecticity-preserving autoencoders}\label{sec:architecture}
Based on this approximation result, we construct in this section a symplecticity-preserving encoder--decoder pair.

\subsection{Symplecticity-preserving decoder} \label{sec:nonlinear_lift}
In view of \Cref{thm:Approximation of nonlinear}, we parameterize a nonlinear symplectic embedding by a structure-preserving neural network of the form
\begin{equation}\label{eq:lift_definition}
\sigma_\theta
=
\mathcal S_{L,\theta_L}\circ\cdots\circ \mathcal S_{1,\theta_1}\circ \mathcal{S_C} ,
\qquad \mathcal{S_C}(z)=S_c z,
\end{equation}
where $\theta = \{\theta_1, \cdots, \theta_L\}$ denotes all the trainable parameters; each $\mathcal S_{i,\theta_i}$ is a symplectic shear map generated by a scalar-valued neural network $V_{i,\theta_i}$.
More precisely, for $x=(q,p)\in\mathbb R^{2n}$ with $q,p\in\mathbb R^n$, we define
\[
\mathcal S_{2m-1,\theta_{2m-1}}(q,p)
=
\bigl(q+\nabla V_{2m-1,\theta_{2m-1}}(p),\, p\bigr),
\]
and
\[
\mathcal S_{2m,\theta_{2m}}(q,p)
=
\bigl(q,\, p+\nabla V_{2m,\theta_{2m}}(q)\bigr).
\]
Here each $V_{i,\theta_i}:\mathbb R^n\to\mathbb R$ is realized by a scalar-output feedforward neural network with smooth activation functions. It is remarked that the update rules for the even and odd shear maps can be entirely interchanged without affecting the underlying symplectic structure.
Since each map $\mathcal S_{i,\theta_i}$ is symplectic and the canonical linear embedding $\mathcal{S_C}(z)=S_c z$ satisfies $S_c^\top J_{2n}S_c=J_{2k}$,
the composite map $\sigma_\theta$ defines a symplectic embedding. In particular, for every $z$ in the domain of $\sigma_\theta$,
\[
(D\sigma_\theta(z))^\top J_{2n}D\sigma_\theta(z)=J_{2k}.
\]
Since the symplectic structure is preserved by construction, we shall refer to $\sigma_\theta$ as a \emph{symplecticity-preserving decoder}.

The above parameterization is also sufficiently expressive to approximate symplectic embeddings on compact contractible domains, as stated in the following theorem.

\begin{theorem}
\label{prop:NN_approx_symp_lift}

Assume that the activation function is smooth and that scalar-output feedforward neural networks with this activation are dense in $C^1(K)$ for every compact set $K\subset\mathbb R^n$. Then
under the conditions of \cref{thm:Approximation of nonlinear}, for every $\varepsilon>0$, there exist an integer $L$ and network parameters $\theta$ such that the symplectic neural embedding $\sigma_\theta$ defined in \eqref{eq:lift_definition} satisfies
\[
\|\sigma-\sigma_\theta\|_{C^0(\mathcal D)}<\varepsilon.
\]
\end{theorem}
It is remarked that the assumptions in \cref{prop:NN_approx_symp_lift} are covered by classical $C^1$-universal approximation results for neural networks. Smooth activation functions such as sigmoid and tanh satisfy the regularity requirement, and the corresponding feedforward neural networks enjoy universal approximation properties on compact sets \cite{cybenko1989approximation,hornik1990universal,shen2021neural}. The proof of \cref{prop:NN_approx_symp_lift} is given in \cref{sec:proofs}.

In the next subsection, we introduce the adjoint projection $\sigma_\theta^+$ associated with $\sigma_\theta$, in a manner consistent with \Cref{def:sym_projection}.

\subsection{Symplecticity-preserving encoder} \label{sec:nonlinear_projection}

Based on the parameterized symplectic embedding established in \Cref{sec:nonlinear_lift}, we define the
nonlinear symplectic projection by:
\begin{equation}\label{eq:proj-def}
    \sigma_\theta^{+} := \mathcal{S_C^\dagger} \circ
\mathcal S^{-1}_{1,\theta_1}\circ\cdots\circ \mathcal S^{-1}_{L,\theta_L}.
\end{equation}
Here $\mathcal S_C^\dagger(x)=S_c^\top x$, and $S_c$ is the canonical embedding matrix \eqref{eq:sc}. In this case, $S_c^\top$ is both the Moore--Penrose inverse of $S_c$ and the symplectic projection matrix associated with $S_c$.
For $x=(q,p)\in\mathbb R^{2n}$ with $q,p\in\mathbb R^n$, the inverse of symplectic shear map $\mathcal S_{i,\theta_{i}}$ can be expressed explicitly as
\[
\mathcal S^{-1}_{2m-1,\theta_{2m-1}}(q,p)
=
\bigl(q-\nabla V_{2m-1,\theta_{2m-1}}(p),\, p\bigr),
\]
and
\[
\mathcal S^{-1}_{2m,\theta_{2m}}(q,p)
=
\bigl(q,\, p-\nabla V_{2m,\theta_{2m}}(q)\bigr).
\]
Since each map $\mathcal S^{-1}_{i,\theta_i}$ is symplectic, and the canonical linear projection $S_c^\top$ satisfies $J_{2k}S_c^\top= S_c^\top J_{2n}$ and $S_c^\top S_c=I_{2k}$,
the composite map $\sigma^{+}_\theta$ satisfies
\begin{equation*}
    \sigma_\theta^+\circ \sigma_\theta = \mathrm{Id}_{\mathbb{R}^{2k}},\quad \text{and} \quad
    J_{2k}\,D\sigma_\theta^+=(D\sigma_\theta)^\top J_{2n},
\end{equation*}
and thus is a symplectic projection of the symplectic
embedding of $\sigma_{\theta}$. Therefore, we refer to $\sigma^{+}_\theta$ as a \emph{symplecticity-preserving encoder}.

\subsection{Symplecticity-preserving autoencoder} \label{sec:retraction_operator}
\begin{figure}
    \centering
    \includegraphics[width=1\linewidth]{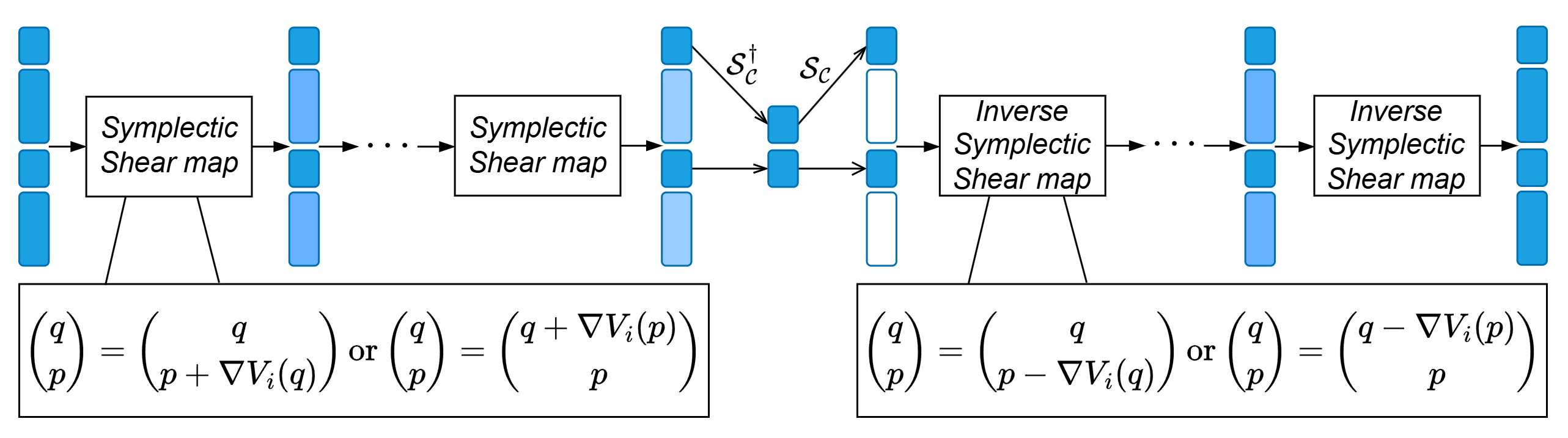}
    \caption{The architecture of the proposed SpAE}
    \label{fig:architecture}
\end{figure}
Let $\sigma_\theta^+:\mathbb R^{2n}\to\mathbb R^{2k}$ and $\sigma_\theta:\mathbb R^{2k}\to\mathbb R^{2n}$ be the symplecticity-preserving encoder and decoder defined above. Define the associated \emph{symplecticity-preserving autoencoder} (SpAE) by
\begin{equation} \label{eq:projection_operator}
    \Pi_\theta := \sigma_\theta \circ \sigma_\theta^{+} = \mathcal S_{L,\theta_L}\circ\cdots\circ \mathcal S_{1,\theta_1}\circ \mathcal{S_C}\circ \mathcal{S_C^\dagger}\circ \mathcal S^{-1}_{1,\theta_1}\circ\cdots\circ \mathcal S^{-1}_{L,\theta_L},
\end{equation}
which is a retraction operator;  see \Cref{fig:architecture} for an illustration.
The retraction operator $\Pi_\theta$ is idempotent satisfying $\Pi_\theta^2=\Pi_\theta$ and $\text{Im}(\Pi_\theta)=\text{Im}(\sigma_\theta)$.
It is noted that the retraction operator $\Pi_{\theta}$ is not a symplectic diffeomorphism on the full space $\mathbb{R}^{2n}$ due to the inherent dimensionality reduction.

When the map $\Pi_\theta $ is restricted to the sub-manifold $\text{Im}(\sigma_\theta)$, it is an identity map. Intuitively, the symplecticity-preserving encoder $\sigma_\theta^{+}$ first projects the sub-manifold $\text{Im}(\sigma_\theta)$ onto the standard $2k$-dimensional space. The decoder $\sigma_\theta$ then embeds this reduced representation back into the $2k$-dimensional sub-manifold.
More precisely,
\[\mathcal S^{-1}_{1,\theta_1}\circ\cdots\circ \mathcal S^{-1}_{L,\theta_L}(x)= y=(Q_k,0,P_k,0),\quad \Pi_\theta (x)= x,\quad \forall x\in\text{Im}(\sigma_\theta).\]

For a $2n$-dimensional $(n>k)$ symplectic manifold, which is approximately a learned $2k$-dimensional symplectic sub-manifold, there also exists a series of symplectic shear maps $\mathcal S^{-1}_{i,\theta_i}$, such that
\[\mathcal S^{-1}_{1,\theta_1}\circ\cdots\circ \mathcal S^{-1}_{L,\theta_L}(x)= y=(Q_k,\varepsilon_{n-k},P_k,\varepsilon_{n-k}),\]
and thus $\Pi_\theta (x)\in \text{Im}(\sigma_\theta)$.
Therefore, $\Pi_\theta$ constitutes a proper retraction onto the learned symplectic submanifold $\text{Im}(\sigma_\theta)$: it maps any point $x\in\mathbb{R}^{2n}$ to a corresponding point on $\text{Im}(\sigma_\theta)$ and satisfies the identity condition $\Pi_\theta|_{\text{Im}(\sigma_\theta)}=\text{id}$.
The retraction can be expressed explicitly as
\begin{equation*}
\begin{aligned}
x\in\mathbb{R}^{2n}
&\xrightarrow{\ \mathcal S^{-1}_{1,\theta_1}\circ\cdots\circ
\mathcal S^{-1}_{L,\theta_L}\ }
\tilde{x}\approx(Q_k,\varepsilon_{n-k},P_k,\varepsilon_{n-k})
\xrightarrow{\ \mathcal{S_C^\dagger}\ }
z=(Q_k,P_k)\in\mathbb{R}^{2k} \\
&\xrightarrow{\ \mathcal{S_C}\ }
\mathcal{S_C}(z)=(Q_k,0,P_k,0)
\xrightarrow{\ \mathcal S_{L,\theta_L}\circ\cdots\circ
\mathcal S_{1,\theta_1}\ }
\hat{x}\in\operatorname{Im}(\sigma_\theta).
\end{aligned}
\end{equation*}

Motivated by the data flow above, we train the retraction by penalizing the component removed by the middle truncation. Let
\(\mathcal T_\theta := \mathcal S_{L,\theta_L}\circ\cdots\circ\mathcal S_{1,\theta_1}\)
and \(\mathcal P := \mathcal S_C\circ\mathcal S_C^\dagger\), so that
\(\Pi_\theta=\mathcal T_\theta\circ\mathcal P\circ\mathcal T_\theta^{-1}\).
For a snapshot matrix \(X\), define \(Y_\theta:=\mathcal T_\theta^{-1}(X)\). Our training objective is then
\begin{equation}\label{eq:loss-inact} \mathcal L_{\rm trunc}(\theta) := \|Y_\theta-\mathcal P Y_\theta\|_F = \|\mathcal T_\theta^{-1}(X) -\mathcal T_\theta^{-1}(\Pi_\theta X)\|_F . \end{equation}
For comparison, the standard autoencoder-type reconstruction loss is
\begin{equation}\label{eq:loss-rec}
\mathcal L_{\rm rec}(\theta)
:=
\|X-\Pi_\theta X\|_F
=
\|\mathcal T_\theta(Y_\theta)
-\mathcal T_\theta(\mathcal P Y_\theta)\|_F .
\end{equation}
If \(\mathcal T_\theta\) is bi-Lipschitz on the relevant region, with
\(\operatorname{Lip}(\mathcal T_\theta),
\operatorname{Lip}(\mathcal T_\theta^{-1})\le C\) for some constant \(C\)
independent of \(\theta\), then the two losses are comparable:
\[
\frac{1}{C}\mathcal L_{\rm trunc}(\theta)
\le
\mathcal L_{\rm rec}(\theta)
\le
C\mathcal L_{\rm trunc}(\theta).
\]

\section{Proofs}\label{sec:proofs}

In this section, we prove the approximation results in above sections. We begin by recalling the basic notions from symplectic geometry used throughout the arguments, then establish the linear factorization theorem, and finally prove the nonlinear approximation theorem by
combining a tubular-neighborhood construction, a relative Darboux argument, and
Turaev's approximation theorem for symplectomorphisms.

\subsection{Preliminaries}\label{sec:preliminaries}
This subsection reformulates the notions introduced earlier in the language of symplectic manifolds.

Let $\omega$ be a de Rham 2-form on a manifold $M$, that is, for each $p \in M$, the map $\omega_p : T_pM \times T_pM \to \mathbb{R}$ is skew-symmetric bilinear on the tangent space to $M$ at $p$, and $\omega_p$ varies smoothly in $p$. We say that $\omega$ is {symplectic} if it is closed (i.e., $d\omega = 0$) and non-degenerate (i.e., for all $p \in M$, the bilinear form $\omega_p$ is non-degenerate). A {symplectic manifold}~\cite{da2008lectures} is a pair $(M, \omega)$ where $M$ is a manifold and $\omega$ is a symplectic form.
In canonical coordinates $x=(q,p)\in\mathbb{R}^{2d}$ with $q=(q_1,\dots,q_d)$ and $p=(p_1,\dots,p_d)$, the standard symplectic form is
\begin{equation*}\label{eq:sympform}
    \omega_0 \;=\; \sum_{i=1}^d dq_i\wedge dp_i,
\end{equation*}
in the standard basis, this form is represented by the matrix $J_{2d}$ as $\omega_0(u, v) = u^\top J_{2d} v$ for $u,v \in \mathbb{R}^{2d}$.

\begin{definition}[Symplectomorphism \cite{da2008lectures}]\label{def:symplectomorphism}
Consider $2d$-dimensional symplectic manifolds $(M_1,\omega_1)$ and $(M_2,\omega_2)$. A diffeomorphism
$\varphi:(M_1,\omega_1)\to(M_2,\omega_2)$ is a \emph{symplectomorphism} if $\varphi^*\omega_2=\omega_1$.
\end{definition}
Specifically, for a differentiable map $\Phi: U \to \mathbb{R}^{2d}$ defined on an open subset $U \subset \mathbb{R}^{2d}$ with Jacobian $D\Phi(x)$, we have
\begin{equation}\label{eq:jac-symp}
\Phi^*\omega_0=\omega_0 \quad\Longleftrightarrow\quad D\Phi(x)^\top J_{2d} D\Phi(x) = J_{2d} \quad\text{for all } x \in U.
\end{equation}

\begin{definition}[Symplectic embedding \cite{cristofaro2018symplectic}]
Let \(U\subset\mathbb R^{2k}\) be open, let \(\omega_{2k}\) be the standard
symplectic form on \(\mathbb R^{2k}\), and let \((M,\omega)\) be a
symplectic manifold. A smooth embedding \(\sigma:U\to M\) is
\emph{symplectic} if \(\sigma^*\omega=\omega_{2k}|_U\). When
\((M,\omega)=(\mathbb R^{2n},\omega_{2n})\), this is equivalent to the
Jacobian condition stated in \Cref{def:symplectic_embedding}.
\end{definition}

\subsection{Proof of \Cref{thm:Factorization of linear}}

\begin{proof}
Since \(A\in\mathbb{R}^{2n\times 2k}\) satisfies
\(A^\top J_{2n}A=J_{2k}\), it follows from~\cite{bendokat2021real} that
there exists \(B\in\mathrm{Sp}(2n,\mathbb{R})\) such that
$A=BS_c$.
By the unit triangular factorization of symplectic
matrices~\cite{jin2022optimal,jin2020unit}, \(B\) can be written as
\begin{equation*}
    B=
\begin{pmatrix}I_n&S_5\\0&I_n\end{pmatrix}
\begin{pmatrix}I_n&0\\S_4&I_n\end{pmatrix}
\begin{pmatrix}I_n&S_3\\0&I_n\end{pmatrix}
\begin{pmatrix}I_n&0\\S_2&I_n\end{pmatrix}
\begin{pmatrix}I_n&S_1\\0&I_n\end{pmatrix},
\end{equation*}
where \(S_i\in\mathbb{R}^{n\times n}\) are symmetric. Substituting this
into \(A=BS_c\) concludes the proof.
\end{proof}

\subsection{Proof of \Cref{thm:Approximation of nonlinear}}

\begin{theorem}[Relative Darboux--Weinstein Theorem
{\cite[Ch.~III, Thm.~15.2]{libermann1987symplectic}}]
\label{thm:relative_darboux}
Let \(M\) be an even-dimensional smooth manifold, let \(N\subset M\) be an
embedded submanifold, and let \(\omega_0\) and \(\omega_1\) be symplectic
forms on \(M\). Assume that the forms are equal on \(N\) in the sense of
\cite[App.~1, \S7.3]{libermann1987symplectic}: for every \(x\in N\), $\omega_0(x)=\omega_1(x)$ as alternating bilinear forms on $T_xM$.
Then there exist open neighborhoods \(U\) and \(U'\) of \(N\) in \(M\) and a
diffeomorphism \(\chi:U\to U'\) such that $\chi|_N=\operatorname{id}_N,\ \chi^*\omega_1=\omega_0$.
Moreover, \(\chi\) may be chosen so that $T_x\chi=\operatorname{id}_{T_xM}, \ x\in N$.

\end{theorem}

\begin{theorem}[Tubular neighborhood Theorem
{\cite[App.~1, Thm.~6.2]{libermann1987symplectic}}]
\label{thm:tubular_neighborhood}
Let \(N\) be an embedded submanifold of a smooth manifold \(M\), and let
\(E\subset T_NM\) be a smooth vector subbundle complementary to \(TN\). Then
there exist an open neighborhood \(W\) of the zero section in \(E\), an open
neighborhood \(U\) of \(N\) in \(M\), and a diffeomorphism
\(\tau:W\to U\) such that \(\tau(0_x)=x\) for every \(x\in N\). Moreover,
\(\tau\) may be chosen so that, under the canonical splitting
\(T_{0_x}E\simeq T_xN\oplus E_x\),
\[
    T_{0_x}\tau(v,\eta)=v+\eta\in T_xM,
    \qquad v\in T_xN,\ \eta\in E_x .
\]
Equivalently, the inverse tubular chart \(\rho=\tau^{-1}:U\to W\) restricts
to the zero section on \(N\) and has identity derivative on the chosen normal
bundle \(E\).

\end{theorem}

We next prove the local canonical decomposition needed for the approximation
theorem. The argument is the standard
relative-Darboux proof of the symplectic neighborhood normal form, with all
hypotheses checked explicitly.

\begin{theorem}[Nonlinear canonical decomposition]
\label{thm:canonical_decomp}
Consider a \(2n\)-dimensional symplectic manifold \((M,\omega)\), let
\(\Omega\subset\mathbb R^{2k}\) be open, and let
\(\sigma:\Omega\to M\) be a smooth symplectic embedding. Let
\(\mathcal K\subset\Omega\) be compact and contained in a contractible open
subset of \(\Omega\).
Then there exist an open neighborhood \(V\) of \(\mathcal{S_C}(\mathcal K)\)
in \(\mathbb{R}^{2n}\), an open neighborhood \(U\) of
\(\sigma(\mathcal K)\) in \(M\), and a symplectomorphism
\(\varphi:V\to U\) such that
\begin{equation*}
    \sigma = \varphi \circ \mathcal{S_C}, \qquad\text{on }\mathcal K .
\end{equation*}
\end{theorem}

\begin{proof}
Choose a contractible open set \(X\) such that
\(\mathcal K\subset X\subset\Omega\). Since \(X\) is an open subset of Euclidean space,
it is a second countable smooth manifold and is paracompact. Let
\[
    \omega_{\mathrm{std}}=\sum_{i=1}^n dq_i\wedge dp_i,
    \
    \omega_X=\sum_{i=1}^k dq_i\wedge dp_i;\
    (M_0,\omega_0)=(\mathbb R^{2n},\omega_{\mathrm{std}}),
    \
    (M_1,\omega_1)=(M,\omega).
\]
Define
\begin{equation}
\label{eq:canonical_embeddings}
    i_0=\mathcal S_C|_X:X\to\mathbb R^{2n},
    \qquad
    i_1=\sigma|_X:X\to M .
\end{equation}
Since \(S_c^\top J_{2n}S_c=J_{2k}\), the map \(i_0\) is a symplectic
embedding, and we have
\begin{equation}
\label{eq:canonical_source_forms}
    i_0^*\omega_0=i_1^*\omega_1=\omega_X .
\end{equation}

For \(j=0,1\), let \(N_jX\) be the symplectic normal bundle of \(i_j\).
The bundle
\(N_0X\to X\) is canonically symplectically trivial, with constant fiber
\[
    F_0=\{0\}^{k}\times\mathbb R^{n-k}
        \times\{0\}^{k}\times\mathbb R^{n-k}.
\]
Since \(X\) is contractible and paracompact, the symplectic frame bundle of
\(N_1X\to X\) is a principal
\(\mathrm{Sp}(2(n-k),\mathbb R)\)-bundle over a contractible paracompact
base, and is therefore trivial by the standard triviality theorem for
principal bundles over contractible paracompact spaces
\cite{husemoller1966fibre}. Thus \(N_1X\) is symplectically trivial. Choosing
symplectic trivializations of \(N_0X\) and \(N_1X\), we obtain a symplectic
vector-bundle isomorphism
\begin{equation}
\label{eq:canonical_normal_isomorphism}
    \Psi:N_0X\to N_1X
\end{equation}
covering the identity on \(X\).

Set \(E=N_0X\). Regard \(E\) as a vector subbundle of \(i_j^*TM_j\) by
\begin{equation}
\label{eq:canonical_normal_inclusions}
    \iota_0:E=N_0X\hookrightarrow i_0^*TM_0,
    \qquad
    \iota_1:E\xrightarrow{\Psi}N_1X\hookrightarrow i_1^*TM_1 .
\end{equation}
Because \(\Psi\) is symplectic, the normal parts satisfy, for all
\(\eta,\zeta\in E_x\),
\begin{equation}
\label{eq:canonical_normal_forms_agree}
    \omega_0(i_0(x))(\iota_0(\eta),\iota_0(\zeta))
    =
    \omega_1(i_1(x))(\iota_1(\eta),\iota_1(\zeta)).
\end{equation}

The maps \(i_0\) and \(i_1\) are embeddings, so \(i_j(X)\subset M_j\) are
embedded submanifolds, identified with \(X\). The bundles \(\iota_j(E)\) are
smooth complements of \(d i_j(TX)\) in \(i_j^*TM_j\). Therefore
\Cref{thm:tubular_neighborhood} applies to \(i_j(X)\subset M_j\) with the
chosen complements \(\iota_j(E)\). Thus there are neighborhoods \(W_j\) of the
zero section in \(E\) and embeddings \(\tau_j:W_j\to M_j\), diffeomorphisms
onto open neighborhoods of \(i_j(X)\), such that
\begin{equation}
\label{eq:canonical_tubular_charts}
    \tau_j(0_x)=i_j(x),
    \qquad
    T_{0_x}\tau_j(v,\eta)=d i_j(v)+\iota_j(\eta),
    \qquad v\in T_xX,\ \eta\in E_x .
\end{equation}

After replacing \(W_0\) and \(W_1\) by their intersection, let \(W\) be a
common open neighborhood of the zero section in \(E\), and define
\begin{equation}
\label{eq:canonical_pulled_forms}
    \widetilde\omega_j=\tau_j^*\omega_j,\qquad j=0,1,
\end{equation}
on \(W\). For \(v,w\in T_xX\) and \(\eta,\zeta\in E_x\), the splitting
\(T_{0_x}E\simeq T_xX\oplus E_x\) and
\eqref{eq:canonical_tubular_charts} give
\[
\begin{aligned}
    \widetilde\omega_j(0_x)((v,\eta),(w,\zeta))
    &=
    \omega_j(i_j(x))
    \bigl(d i_j(v)+\iota_j(\eta),
          d i_j(w)+\iota_j(\zeta)\bigr).
\end{aligned}
\]
The mixed terms vanish because \(\iota_j(E)=N_jX\) is the
\(\omega_j\)-orthogonal complement of \(d i_j(TX)\). The tangent terms agree
by \eqref{eq:canonical_source_forms}, and the normal terms agree by
\eqref{eq:canonical_normal_forms_agree}. Hence
\begin{equation}
\label{eq:canonical_zero_section_agreement}
    \widetilde\omega_0(0_x)=\widetilde\omega_1(0_x)
    \qquad\text{as bilinear forms on }T_{0_x}E,\quad x\in X .
\end{equation}
This is the strong equality on the zero section required by
\Cref{thm:relative_darboux}, not merely equality after restriction to
\(T_xX\).

Apply \Cref{thm:relative_darboux} to the even-dimensional manifold \(W\), the
zero section \(X\subset W\), and the symplectic forms
\(\widetilde\omega_0\) and \(\widetilde\omega_1\). By
\eqref{eq:canonical_zero_section_agreement}, its hypotheses are satisfied.
Thus there exist neighborhoods \(W'\) and \(W''\) of the zero section in \(W\)
and a diffeomorphism \(\chi:W'\to W''\) such that
\begin{equation}
\label{eq:canonical_darboux_correction}
    \chi|_X=\operatorname{id}_X,
    \qquad
    \chi^*\widetilde\omega_1=\widetilde\omega_0 .
\end{equation}
Again, \Cref{thm:relative_darboux} does not require \(X\) to be compact:
it produces open neighborhoods of the zero section over all of \(X\). These
neighborhoods therefore contain the compact subset \(\{0_z:z\in\mathcal K\}\).

Set
\begin{equation}
\label{eq:canonical_varphi_def}
    V=\tau_0(W'),\qquad U=\tau_1(W''),\qquad
    \varphi=\tau_1\circ\chi\circ\tau_0^{-1}.
\end{equation}
Then \(V\) is an open neighborhood of \(\mathcal S_C(\mathcal K)\) in
\(\mathbb R^{2n}\), because \(\tau_0(0_z)=i_0(z)=\mathcal S_C(z)\) for
\(z\in\mathcal K\). Similarly, \(U\) is an open neighborhood of
\(\sigma(\mathcal K)\) in \(M\), because
\(\tau_1(0_z)=i_1(z)=\sigma(z)\). Since \(\tau_0\), \(\chi\), and \(\tau_1\)
are diffeomorphisms onto their images, \(\varphi:V\to U\) is a
diffeomorphism. Moreover,
\[
\begin{aligned}
    \varphi^*\omega_1
    &=(\tau_0^{-1})^*\chi^*(\tau_1^*\omega_1)
        &&\text{by the definition of \(\varphi\) in
        \eqref{eq:canonical_varphi_def}}  \\
    &=(\tau_0^{-1})^*\chi^*\widetilde\omega_1
        &&\text{by \eqref{eq:canonical_pulled_forms}}  \\
    &=(\tau_0^{-1})^*\widetilde\omega_0
        &&\text{by \eqref{eq:canonical_darboux_correction}}  \\
    &=(\tau_0^{-1})^*(\tau_0^*\omega_0)
        &&\text{by \eqref{eq:canonical_pulled_forms}}  \\
    &=\omega_0|_V
        &&\text{because \(\tau_0:W'\to V\) is a diffeomorphism.}
\end{aligned}
\]
Thus \(\varphi\) is a symplectomorphism from \((V,\omega_0|_V)\) to
\((U,\omega_1|_U)\), by \Cref{def:symplectomorphism}.
Therefore, for every \(z\in\mathcal K\),
\[
\begin{aligned}
    \varphi(\mathcal S_C(z))
    &=\varphi(i_0(z))
        &&\text{by \(i_0=\mathcal S_C|_X\) in
        \eqref{eq:canonical_embeddings}}  \\
    &=\tau_1\!\left(\chi\!\left(\tau_0^{-1}(i_0(z))\right)\right)
        &&\text{by the definition of \(\varphi\) in
        \eqref{eq:canonical_varphi_def}}  \\
    &=\tau_1(\chi(0_z))
        &&\text{because \(\tau_0(0_z)=i_0(z)\), by
        \eqref{eq:canonical_tubular_charts}}  \\
    &=\tau_1(0_z)
        &&\text{because \(\chi|_X=\operatorname{id}_X\), by
        \eqref{eq:canonical_darboux_correction}}  \\
    &=i_1(z)
        &&\text{because \(\tau_1(0_z)=i_1(z)\), by
        \eqref{eq:canonical_tubular_charts}}  \\
    &=\sigma(z)
        &&\text{by \(i_1=\sigma|_X\) in
        \eqref{eq:canonical_embeddings}.}
\end{aligned}
\]
Hence
\(\sigma=\varphi\circ\mathcal S_C\) on \(\mathcal K\), as required.
\end{proof}

Since \(\varphi\) is a symplectomorphism, we recall the following approximation
theorem for symplectomorphisms, due to Turaev~\cite{turaev2003polynomial}.

\begin{lemma}[Approximation of symplectomorphisms]
\label{thm:approximation_of_symplectomorphisms}
Given any \(C^r\)-smooth symplectic diffeomorphism
\(\varphi:U\to\varphi(U)\subset\mathbb R^{2n}\), for any compact set
\(C\subset U\) and any \(\varepsilon>0\), there exist polynomial symplectic
shear maps \(S_1,\ldots,S_L\) such that
\[
    \|\varphi-S_L\circ\cdots\circ S_1\|_{C^r(C)}<\varepsilon .
\]
Here, in canonical coordinates
\((q,p)\in \mathbb{R}^n\times\mathbb{R}^n\), a polynomial symplectic shear map
is a map of one of the following forms:
\[
    (q,p)\mapsto (q,p+\nabla P(q)),
    \qquad
    (q,p)\mapsto (q+\nabla Q(p),p),
\]
where \(P\) and \(Q\) are polynomials.
\end{lemma}

\begin{proof}
This is the intermediate approximation result established in the proof of
\cite[Theorem~2]{turaev2003polynomial}.
\end{proof}

We are now ready to prove \Cref{thm:Approximation of nonlinear}.
\begin{proof}[Proof of \Cref{thm:Approximation of nonlinear}]
By \Cref{thm:canonical_decomp}, applied with \(\mathcal K=\mathcal D\), there
exist an open set \(V\subset\mathbb R^{2n}\) containing
\(\mathcal S_C(\mathcal D)\) and a symplectomorphism
\(\phi:V\to\phi(V)\) such that
\(\sigma(z)=\phi(\mathcal S_C(z))\) for \(z\in\mathcal D\).
Since \(\mathcal S_C(\mathcal D)\) is compact and \(V\) is open, choose a
compact set \(K\) such that
\(\mathcal S_C(\mathcal D)\subset K\subset V\). Applying
\Cref{thm:approximation_of_symplectomorphisms} to \(\phi\) on \(K\) in the
\(C^r\)-topology for any fixed \(r\ge 1\), and retaining only the \(C^0\)
estimate, we obtain that, for every \(\varepsilon>0\), there exist
polynomial symplectic shear maps
\(S_1,\ldots,S_L\) such that
\[
    \sup_{x\in K}
    \|\phi(x)-S_L\circ\cdots\circ S_1(x)\|<\varepsilon .
\]
Here each \(S_i\) has the form
\((q,p)\mapsto(q+\nabla V_i(p),p)\) or
\((q,p)\mapsto(q,p+\nabla V_i(q))\), for some polynomial potential
\(V_i:\mathbb R^n\to\mathbb R\). Inserting identity shears generated by the
zero potential if necessary, we may assume that the two shear types alternate
according to the parity convention in the statement. Finally, since
\(\mathcal S_C(\mathcal D)\subset K\) and
\(\sigma=\phi\circ \mathcal S_C\) on \(\mathcal D\), we conclude that
\[
\begin{aligned}
    \|\sigma-S_L\circ\cdots\circ S_1\circ \mathcal S_C\|_{C^0(\mathcal D)}
    &=
    \sup_{z\in\mathcal D}
    \|\phi(\mathcal S_C(z))-S_L\circ\cdots\circ S_1(\mathcal S_C(z))\| \\
    &\le
    \sup_{x\in K}
    \|\phi(x)-S_L\circ\cdots\circ S_1(x)\|
    <\varepsilon.
\end{aligned}
\]
\end{proof}

\subsection{Proof of \Cref{prop:NN_approx_symp_lift}}
\begin{proof}[Proof of \Cref{prop:NN_approx_symp_lift}]
By \Cref{thm:Approximation of nonlinear}, there exist
\(L\in\mathbb N\) and polynomial, hence smooth, potentials
\(V_1,\ldots,V_L:\mathbb R^n\to\mathbb R\) such that
\begin{equation}\label{eq:estimate1}
\|\sigma-\widetilde\sigma_L\|_{C^0(\mathcal D)}
    <\frac{\varepsilon}{2},\quad \widetilde\sigma_L
    :=
    \mathcal S_L\circ\cdots\circ\mathcal S_1\circ \mathcal{S_C}.
\end{equation}
Note that each smooth shear map can be uniformly approximated on
compact sets by a neural shear map of the same type. Using Lemma 1 in \cite{zhu2022approximation} we obtain that \(\widetilde\sigma_L\) can be uniformly approximated by $\sigma_\theta$.
Applying such result to the compact set
\(\mathcal{S_C}(\mathcal D)\) with accuracy \(\varepsilon/2\), we obtain
neural shear maps such that the corresponding
neural embedding \(\sigma_\theta\) satisfies
\(
    \|\widetilde\sigma_L-\sigma_\theta\|_{C^0(\mathcal D)}
    <\frac{\varepsilon}{2}.
\)
Combining this estimate with \eqref{eq:estimate1} gives
\[
    \|\sigma-\sigma_\theta\|_{C^0(\mathcal D)}
    \le
    \|\sigma-\widetilde\sigma_L\|_{C^0(\mathcal D)}
    +
    \|\widetilde\sigma_L-\sigma_\theta\|_{C^0(\mathcal D)}
    <\varepsilon,
\]
which concludes the proof.
\end{proof}

\section{Numerical results}\label{sec:experiments}
In this section, we evaluate the proposed SpAE framework on three representative
high-dimensional Hamiltonian systems: a 1D crystal-lattice model for
classical molecular dynamics, a multi-particle model of tokamak electromagnetic
dynamics, and a two-stream instability model. These benchmarks cover both lattice
and particle Hamiltonian dynamics and are used to assess the ability of SpAE to
learn low-dimensional nonlinear symplectic representations, achieve accurate
reconstruction relative to existing reduction methods, and support long-horizon
prediction after learning the reduced dynamics in the latent space.

For the numerical implementation of SpAE, we parameterize the scalar potential \(V_i\) in each shear layer by a single-hidden-layer neural network,
\begin{equation*}
    V_i(x)
    =
    a_i^{\top}\rho(K_i x+b_i),
    \qquad x\in\mathbb{R}^n .
\end{equation*}
Here \(K_i\in\mathbb{R}^{w_i\times n}\), \(a_i,b_i\in\mathbb{R}^{w_i}\)
are trainable parameters, \(w_i\) is the width of the \(i\)-th potential
network, and \(\rho\) is an activation function applied componentwise. We have
\[
    \nabla V_i(x)
    =
    K_i^{\top}\operatorname{diag}(a_i)\rho'(K_i x+b_i),
\]
which is exactly the gradient-module form of G-SympNets
\cite{Jin2020SympNets}. After learning the latent representation, we employ an HNN~\cite{greydanus2019hamiltonian} to model the reduced Hamiltonian dynamics for SpAE and the common reduction baselines. We compare SpAE with WSAE, H\'{e}non-ROM, COT, cSVD, NLP, and POD.
The linear symplectic reduction baselines COT, cSVD, and NLP are associated with the proper symplectic decomposition framework~\cite{peng2016symplectic}, whereas POD~\cite{kunisch2002galerkin} serves as a classical non-symplectic linear baseline. Among the nonlinear baselines, weakly symplectic autoencoder (WSAE)~\cite{buchfink2023symplectic} augments the reconstruction objective with a penalty on the decoder-Jacobian residual of the canonical symplecticity condition, so symplecticity is not guaranteed globally by the architecture.
We additionally evaluate H\'{e}non-ROM~\cite{chen2025reduced}, which combines a H\'{e}nonNet-based symplectic encoder--decoder augmented with G-reflector layers and a H\'{e}nonNet latent flow map.

For each system, multiple trajectories are generated by applying independent Gaussian perturbations with amplitude $\epsilon$ to a common initial condition and are divided into training, validation, and one blind target trajectory. The training and validation trajectories are used for model fitting and selection, respectively. The blind target is excluded from reducer and latent-dynamics training, normalization, hyperparameter tuning, and model selection, and is used only to evaluate reconstruction and long-horizon prediction. All tabulated errors are computed over the complete target trajectory: the absolute error is the Frobenius norm of the difference between the target and reconstructed or predicted data matrices, and the relative error normalizes it by the target-data norm. The figures report per-snapshot $L_2$ errors on the same trajectory. All neural networks are implemented in PyTorch and trained with Adam, with the principal settings summarized in \Cref{tab:exp_params}.

\begin{table}[htbp]
\centering
\caption{Experimental settings for the three Hamiltonian systems.}
\label{tab:exp_params}
\small
\setlength{\tabcolsep}{4pt}
\renewcommand{\arraystretch}{1.22}
\begin{tabularx}{0.8\textwidth}{
    >{\raggedright\arraybackslash}p{0.24\textwidth}
    YYY
}
\toprule
Parameter
& Crystal lattice
& Tokamak
& Two-stream \\
\midrule

Dimension reduction
& $\mathbb{R}^{1000}\to\mathbb{R}^{8}$
& $\mathbb{R}^{3000}\to\mathbb{R}^{6}$
& $\mathbb{R}^{2000}\to\mathbb{R}^{10}$ \\

Trajectory split
& $80/10/1$
& $52/3/1$
& $40/8/1$ \\

Snapshots
& $1000$
& $5000$
& $1001$ \\

Snapshot indices
& $0$--$999$
& $0$--$4999$
& $0$--$1000$ \\

Time step
& $\Delta t=0.01$
& $\Delta t=0.01$
& $\Delta t=0.01$ \\

\midrule

SpAE reducer
& $8\times192$ / sigmoid
& $6\times256$ / tanh
& $8\times256$ / sigmoid \\

WSAE hidden widths
& $512/256/128$
& $512/256/128$
& $768/384/192$ \\

WSAE activation / $\alpha$
& ELU / $0.50$
& ELU / $0.50$
& ELU / $0.70$ \\

SpAE latent HNN
& $6\times128$ / tanh
& $6\times128$ / tanh
& $6\times160$ / tanh \\

Latent integrator
& SRK4
& SRK4
& SRK4 \\

\bottomrule
\end{tabularx}

\vspace{0.4ex}
\begin{minipage}{0.94\textwidth}
\raggedright
\footnotesize
SRK4 denotes the implicit fourth-order symplectic Runge--Kutta scheme.
\end{minipage}
\end{table}

\subsection{Crystal Lattice Model of Classical Molecular Dynamics}

We consider a 1D nonlinear crystal lattice model consisting of $N$ particles with periodic boundary conditions \cite{bajars2023dimensionality,bajars2025structure}. The dynamics of the system are governed by the dimensionless and separable Hamiltonian, which represents the total energy as the sum of kinetic energy, on-site potential energy, and nearest-neighbor interaction potential energy \cite{archilla2023spectral,archilla2019pterobreathers,bajars2022data,dou2010breathers}:
\begin{equation}
    H(\mathbf{q}, \mathbf{p}) = \sum_{n=0}^{N-1} \left( \frac{1}{2} p_n^2 + U(q_n) + V(r_n) \right),
    \label{eq:lattice-hamiltonian}
\end{equation}
where $q_n$ and $p_n$ denote the displacement from the equilibrium position and the momentum of the $n$-th particle, respectively. The term $r_n = 1 + q_{n+1} - q_n$ represents the distance between neighboring particles. The on-site potential $U(q_n)$ is modeled by the periodic Frenkel-Kontorova potential \cite{braun1998nonlinear}; and the interaction between adjacent particles is described by the scaled Lennard-Jones potential
\begin{equation}
    U(q_n) = U_0 \left( 1 - \cos(2\pi q_n) \right),
\quad
    V(r_n) = V_0 \left[ \left( \frac{\eta}{r_n} \right)^{12} - 2 \left( \frac{\eta}{r_n} \right)^{6} \right],
\end{equation}
where $U_0$ is the depth of the on-site potential well; $\eta$ is the scaled lattice constant and $V_0$ determines the interaction strength. This highly nonlinear system supports the existence of Discrete Breathers (DBs), which are spatially localized and time-periodic excitations.

\begin{figure}[htbp]
    \centering
    \includegraphics[width=\linewidth]{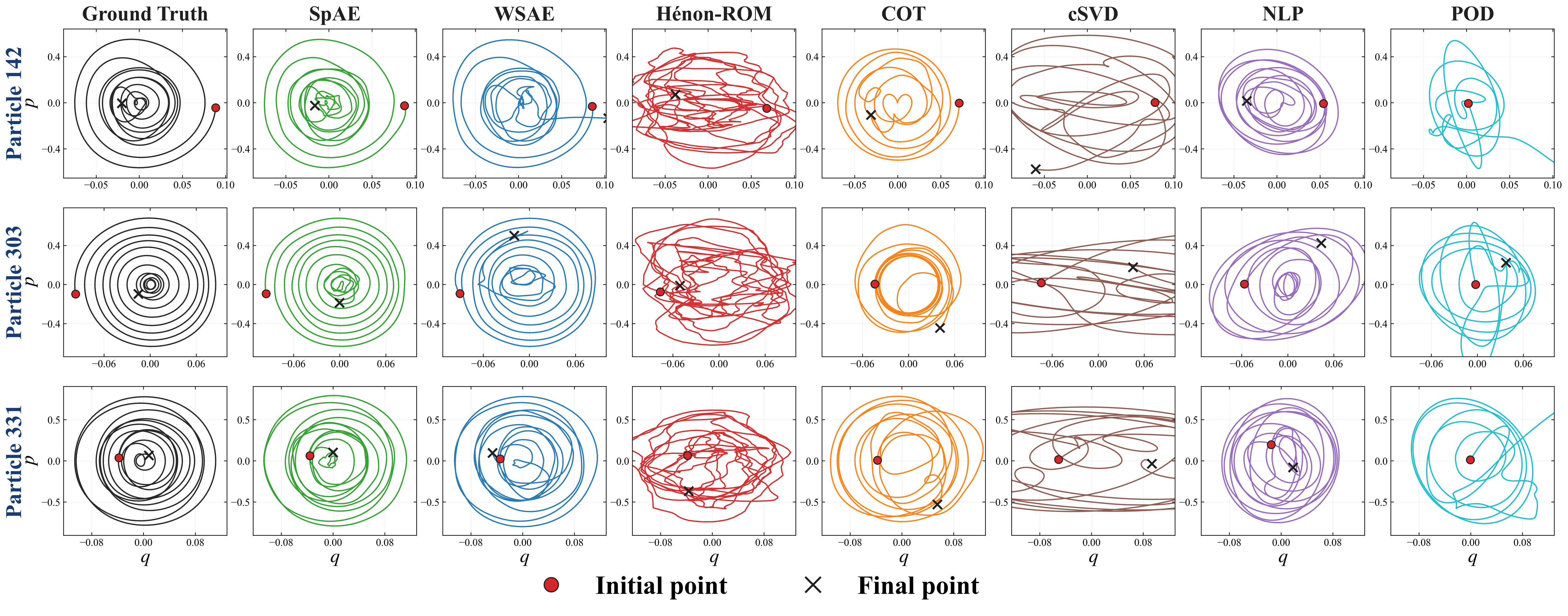}
    \caption{Phase-space portraits of particles 142, 303, and 331 on the blind crystal lattice trajectory. Red dots and black crosses mark the initial and final states, respectively.}
    \label{fig:latice trajectory}
\end{figure}

\begin{table}[htbp]
\centering
\caption[Errors on the Blind Crystal Lattice Trajectory]
{Reconstruction and Prediction Errors on the Blind Crystal Lattice Trajectory}
\label{tab:lattice_errors}
\resizebox{0.8\textwidth}{!}{
\begin{tabular}{lcccc}
\toprule
\multirow{2}{*}{Method}
& \multicolumn{2}{c}{Reconstruction (Stage 1)}
& \multicolumn{2}{c}{Full Prediction (Stage 2)} \\
\cmidrule(lr){2-3} \cmidrule(lr){4-5}
& Absolute Error
& Relative Error
& Absolute Error
& Relative Error \\
\midrule
SpAE (Ours)
& $9.1169 \times 10^{0}$
& $4.4341 \times 10^{-2}$
& $\mathbf{3.6666 \times 10^{1}}$
& $\mathbf{1.7833 \times 10^{-1}}$ \\

WSAE
& $\mathbf{9.0470 \times 10^{0}}$
& $\mathbf{4.4001 \times 10^{-2}}$
& $1.6298 \times 10^{2}$
& $7.9266 \times 10^{-1}$ \\

H\'{e}non-ROM
& $9.3606 \times 10^{0}$
& $4.5526 \times 10^{-2}$
& $2.1661 \times 10^{2}$
& $1.0535 \times 10^{\phantom{-}0}$ \\

COT
& $1.2642 \times 10^{2}$
& $6.1484 \times 10^{-1}$
& $2.3106 \times 10^{2}$
& $1.1238 \times 10^{\phantom{-}0}$ \\

cSVD
& $1.1803 \times 10^{2}$
& $5.7405 \times 10^{-1}$
& $2.3473 \times 10^{2}$
& $1.1416 \times 10^{\phantom{-}0}$ \\

NLP
& $4.4166 \times 10^{1}$
& $2.1481 \times 10^{-1}$
& $1.7404 \times 10^{2}$
& $8.4648 \times 10^{-1}$ \\

POD
& $3.5410 \times 10^{1}$
& $1.7222 \times 10^{-1}$
& $5.9601 \times 10^{2}$
& $2.8987 \times 10^{\phantom{-}0}$ \\
\bottomrule
\end{tabular}
}
\end{table}

\begin{figure}[htbp]
    \centering
    \includegraphics[width=0.8\linewidth]{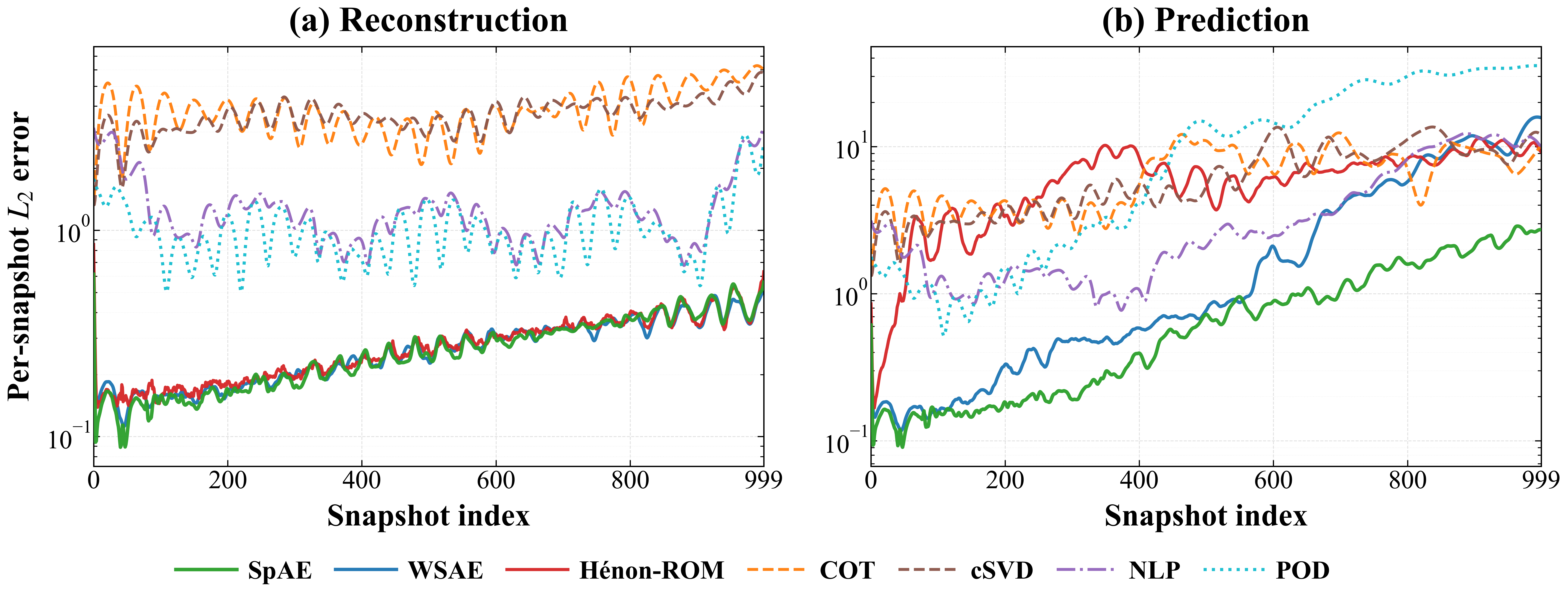}
    \caption{Per-snapshot $L_2$ errors on a logarithmic scale for the blind crystal lattice trajectory. Panel (a) shows the reconstruction error, and panel (b) shows the full-prediction error.}
    \label{fig:lattice error}
\end{figure}

We simulate a one-dimensional crystal lattice of $N=500$ particles, yielding a phase space in $\mathbb{R}^{1000}$. Perturbations of amplitude $\epsilon=10^{-3}$ generate 91 trajectories, each containing 1000 snapshots with $\Delta t=0.01$. The state is reduced to $2k=8$, and the trajectories are split into 80 for training, 10 for validation, and one blind target for final evaluation.

For this experiment only, we additionally evaluate H\'{e}non-ROM~\cite{chen2025reduced}. Its reducer consists of two width-$300$ H\'{e}non layers followed by four G-reflector layers, while its latent flow uses two width-$30$ H\'{e}non layers. All scalar-potential networks use ELU activations.

The reconstruction results in \Cref{tab:lattice_errors} and the reconstruction panel of \cref{fig:lattice error} show that SpAE, WSAE, and H\'{e}non-ROM achieve similarly accurate compression from $\mathbb{R}^{1000}$ to $\mathbb{R}^{8}$. The three nonlinear reducers consistently outperform all linear baselines: their errors are more than one order of magnitude lower than those of COT and cSVD and are also clearly lower than those of POD and NLP. The results indicate that nonlinear reduction is more effective than linear methods in capturing the underlying low-dimensional structure.

The prediction results in \Cref{tab:lattice_errors} and the prediction panel of \cref{fig:lattice error} show that the latent-dynamics model trained on the representation learned by SpAE yields the lowest long-horizon prediction error among all tested reduction methods. Its prediction error is more than one order of magnitude lower than that obtained with POD and is also lower than those obtained with the other baselines. In \cref{fig:latice trajectory}, the trajectories predicted using the SpAE representation agree most closely with the ground-truth trajectories, whereas those obtained using the other representations exhibit larger deviations.

\subsection{Charged Particles in Tokamak Electromagnetic Fields}

Charged-particle motion in tokamak magnetic confinement exhibits rapid gyromotion in a spatially nonuniform axisymmetric magnetic field~\cite{zhang2016explicit,tao2016explicit}. We consider an ensemble of noninteracting particles evolving in a prescribed, time-independent axisymmetric magnetic field and a static electrostatic field.

In normalized units, an ensemble of $N$ charged particles with canonical positions $\boldsymbol q_i=(x_i,y_i,z_i)$ and momenta $\boldsymbol p_i$, for $i=1,\ldots,N$, is governed by the total Hamiltonian
\begin{equation}
H(\boldsymbol q,\boldsymbol p)
=
\sum_{i=1}^{N}
\left[
\frac{1}{2}
\left\lVert
\boldsymbol p_i-\boldsymbol A(\boldsymbol q_i)
\right\rVert_2^2
+
\Phi(\boldsymbol q_i)
\right].
\label{eq:tokamak-hamiltonian}
\end{equation}
The magnetic vector potential is
\begin{equation}
\boldsymbol{A}(x,y,z)
=
\left(
\frac{2xz-y\rho_{\mathrm m}^2}{4R^2},
\frac{2yz+x\rho_{\mathrm m}^2}{4R^2},
-\frac{1}{4}\ln R^2
\right),
\
R=\sqrt{x^2+y^2},
\
\rho_{\mathrm m}^2=(R-R_{\mathrm m})^2+z^2,
\label{eq:tokamak-vector-potential}
\end{equation}
with magnetic-axis major radius $R_{\mathrm m}=1$. The experiment includes the static coupled quartic electric potential
$
\Phi(x,y,z)
=\frac{k_2}{2}\rho_{\Phi}^2
+\frac{k_4}{4}\rho_{\Phi}^4
+k_cXyz,
$
where $X=x-R_{\Phi}$, $\rho_{\Phi}^2=X^2+y^2+z^2$, $R_{\Phi}=1.25$, $k_2=0.05$, $k_4=0.20$, and $k_c=0.05$. Since $\Phi$ is time independent, the continuous system is autonomous and conserves the total Hamiltonian along exact trajectories.

\begin{figure}[htbp]
    \centering
    \includegraphics[width=1\linewidth]{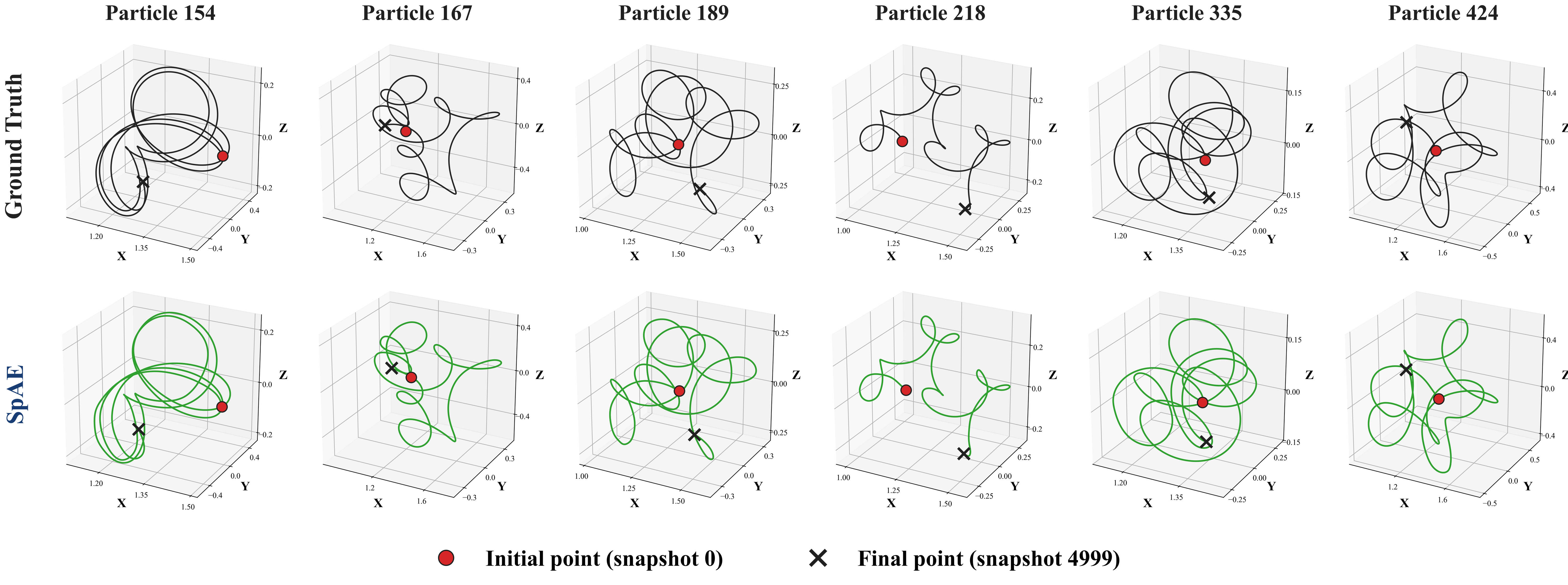}
\caption{Three-dimensional position trajectories of several particles over all 5,000 snapshots of the blind Tokamak trajectory. Red dots and black crosses mark snapshots 0 and 4999, respectively.}
    \label{fig:pred_tokamak}
\end{figure}

\begin{table}[htbp]
\centering
\caption[Reconstruction and Prediction Errors for the Tokamak System]
{Reconstruction and Prediction Errors on the Blind Tokamak Trajectory}
\label{tab:tokamak_errors}
\resizebox{0.8\textwidth}{!}{
\begin{tabular}{lcccc}
\toprule
\multirow{2}{*}{\textbf{Method}}
& \multicolumn{2}{c}{\textbf{Reconstruction (Stage 1)}}
& \multicolumn{2}{c}{\textbf{Prediction (Stage 2)}} \\
\cmidrule(lr){2-3} \cmidrule(lr){4-5}
& \textbf{Absolute Error}
& \textbf{Relative Error}
& \textbf{Absolute Error}
& \textbf{Relative Error} \\
\midrule

\textbf{SpAE (Ours)}
& $6.3670 \times 10^{\phantom{-}1}$
& $2.9447 \times 10^{-2}$
& $\mathbf{6.4538 \times 10^{\phantom{-}1}}$
& $\mathbf{2.9849 \times 10^{-2}}$ \\

WSAE
& $\mathbf{6.3046 \times 10^{\phantom{-}1}}$
& $\mathbf{2.9159 \times 10^{-2}}$
& $7.9430 \times 10^{\phantom{-}1}$
& $3.6736 \times 10^{-2}$ \\

COT
& $3.4350 \times 10^{\phantom{-}2}$
& $1.5887 \times 10^{-1}$
& $3.5467 \times 10^{\phantom{-}2}$
& $1.6404 \times 10^{-1}$ \\

cSVD
& $2.9548 \times 10^{\phantom{-}2}$
& $1.3666 \times 10^{-1}$
& $3.9321 \times 10^{\phantom{-}2}$
& $1.8186 \times 10^{-1}$ \\

NLP
& $2.4936 \times 10^{\phantom{-}2}$
& $1.1533 \times 10^{-1}$
& $3.8928 \times 10^{\phantom{-}2}$
& $1.8004 \times 10^{-1}$ \\

POD
& $2.0290 \times 10^{\phantom{-}2}$
& $9.3843 \times 10^{-2}$
& $4.0209 \times 10^{\phantom{-}2}$
& $1.8597 \times 10^{-1}$ \\

\bottomrule
\end{tabular}
}
\end{table}

\begin{figure}[htbp]
    \centering
    \includegraphics[width=0.8\linewidth]{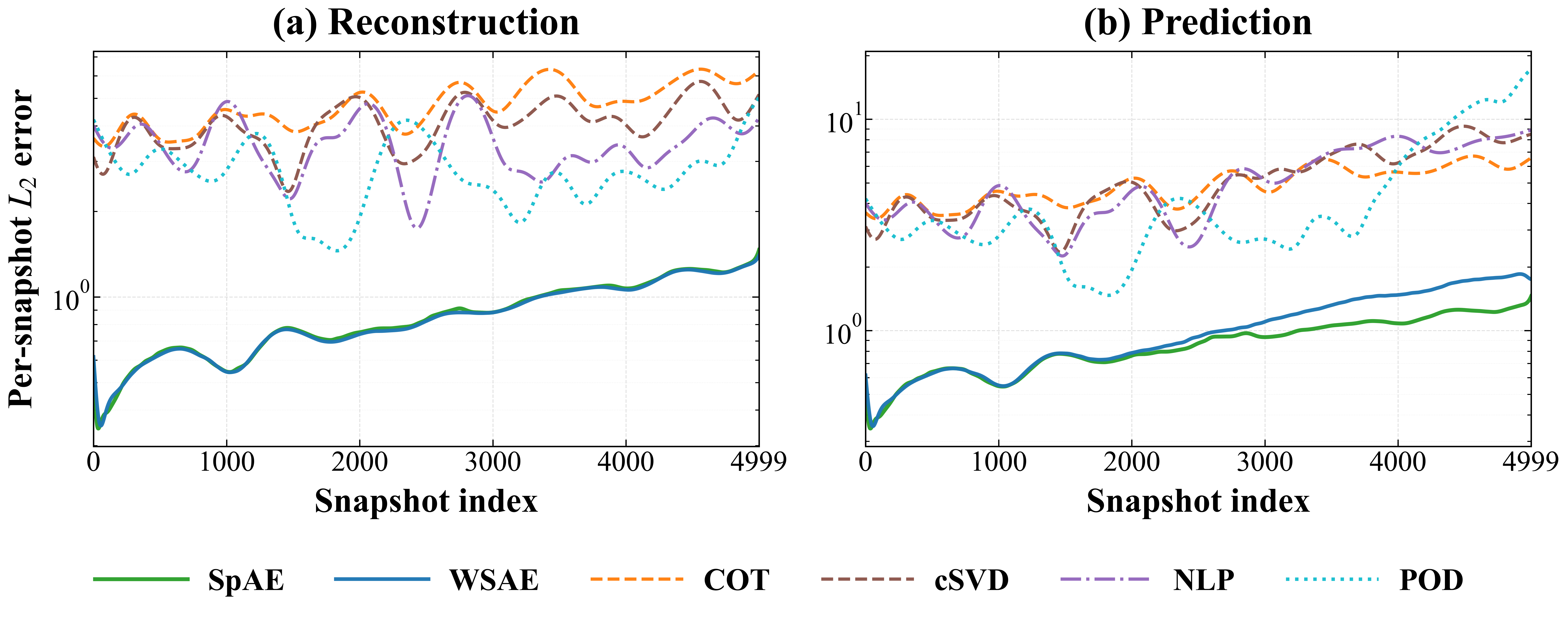}
    \caption{Per-snapshot $L_2$ errors on a logarithmic scale for the blind Tokamak trajectory. Panel (a) shows the reconstruction error, and panel (b) shows the full-prediction error.}
    \label{fig:Error_tokamak}
\end{figure}

We simulate 500 charged particles, giving a phase space in
$\mathbb{R}^{3000}$. Ground-truth trajectories are generated by the
second-order implicit midpoint method with $\Delta t=0.01$. Microscopic
perturbations of amplitude $\epsilon=10^{-2}$ produce 56 trajectories,
each containing 5000 snapshots, and the state is reduced to $2k=6$.
The split contains 52 training trajectories, three validation
trajectories, and one blind target trajectory for
final evaluation.

The reconstruction results in \Cref{tab:tokamak_errors} and the reconstruction panel of \cref{fig:Error_tokamak} indicate that the $\mathbb{R}^{3000}$ Tokamak states can be represented accurately in a six-dimensional latent space by the nonlinear reducers. SpAE and WSAE yield comparable errors in the $10^{-2}$ range. Although POD remains in the same order of magnitude, its error is clearly higher, while COT, cSVD, and NLP produce errors in the $10^{-1}$ range. This ordering confirms the strong reconstruction performance of the nonlinear models under substantial dimensionality reduction.

The prediction results in \Cref{tab:tokamak_errors} and the prediction panel of \cref{fig:Error_tokamak} show that SpAE yields the lowest long-horizon prediction error among all reduction methods. Its error is lower than that of WSAE within the $10^{-2}$ range, whereas the errors of all linear methods lie in the $10^{-1}$ range. In \cref{fig:pred_tokamak}, the trajectories predicted using SpAE agree closely with the ground-truth trajectories over 5,000 snapshots, indicating that SpAE provides an effective reduced representation for long-horizon prediction.

\subsection{Two-Stream Instability Model}

We consider a 1D periodic $N$-particle two-stream model with a Gaussian-screened reciprocal-space interaction~\cite{cheng1976integration,de1980simulation,escande2016n}. The Hamiltonian is
\begin{equation}
H(\mathbf{q},\mathbf{p})
=
\sum_{i=1}^{N}\frac{p_i^2}{2m}
+
\frac{e^2}{4\pi\varepsilon_0}
\sum_{\ell=1}^{M_{\mathrm F}}\exp\!\left(
-\frac{k_\ell^2}{4\alpha_{\mathrm{ew}}^2}
\right)
\left[
\left(\sum_{j=1}^{N}\cos(k_\ell q_j)\right)^2
+
\left(\sum_{j=1}^{N}\sin(k_\ell q_j)\right)^2
\right].
\end{equation}
Here, $M_{\mathrm F}$ denotes the number of retained positive Fourier modes, $k_\ell=2\pi\ell/L$ is the wave number of the $\ell$-th mode, and $\alpha_{\mathrm{ew}}>0$ controls the exponential attenuation of high-wave-number modes.

\begin{figure}[htbp]
    \centering
\includegraphics[width=1\linewidth]{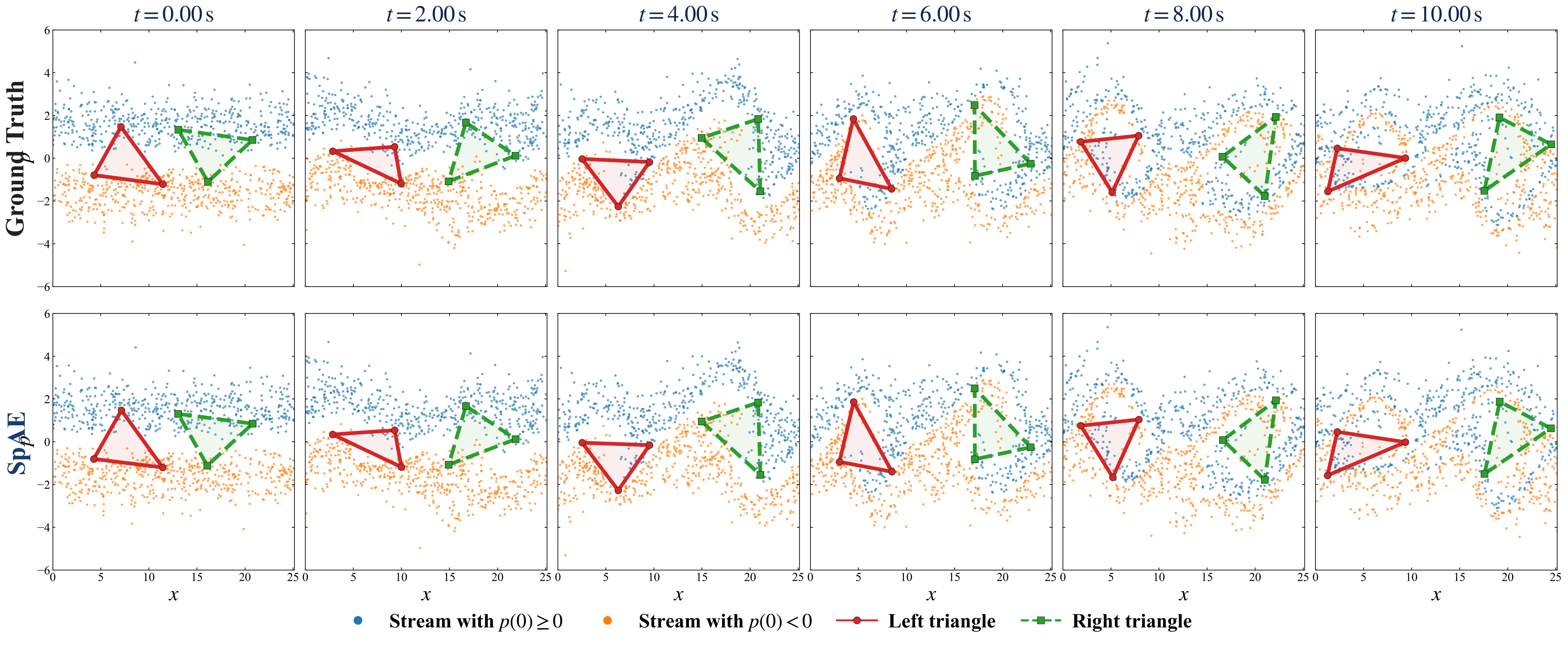}
    \caption{Phase-space evolution of the two-stream instability at $t=0.00$, $2.00$, $4.00$, $6.00$, $8.00$, and $10.00\,\mathrm{s}$. The top and bottom rows show the ground truth and SpAE predictions, respectively. Blue and orange points denote particles with $p(0)\geq 0$ and $p(0)<0$, respectively. The solid red triangle connects particles 49, 196, and 936, while the dashed green triangle connects particles 222, 307, and 589.}
    \label{fig:pred_twostream}
\end{figure}

\begin{table}[htbp]
\centering
\caption{Reconstruction and prediction errors on the blind two-stream trajectory.}
\label{tab:twostream_errors}
\resizebox{0.8\textwidth}{!}{
\begin{tabular}{lcccc}
\toprule
\multirow{2}{*}{\textbf{Method}}
& \multicolumn{2}{c}{\textbf{Reconstruction (Stage 1)}}
& \multicolumn{2}{c}{\textbf{Prediction (Stage 2)}} \\
\cmidrule(lr){2-3} \cmidrule(lr){4-5}
& \textbf{Absolute Error}
& \textbf{Relative Error}
& \textbf{Absolute Error}
& \textbf{Relative Error} \\
\midrule
\textbf{SpAE (Ours)}
& $\mathbf{2.1175 \times 10^{\phantom{-}1}}$
& $\mathbf{1.2431 \times 10^{-3}}$
& $\mathbf{2.2240 \times 10^{\phantom{-}1}}$
& $\mathbf{1.3056 \times 10^{-3}}$ \\

WSAE
& $3.3703 \times 10^{\phantom{-}1}$
& $1.9785 \times 10^{-3}$
& $3.5844 \times 10^{\phantom{-}1}$
& $2.1042 \times 10^{-3}$ \\

COT
& $4.3174 \times 10^{\phantom{-}2}$
& $2.5345 \times 10^{-2}$
& $4.6770 \times 10^{\phantom{-}2}$
& $2.7457 \times 10^{-2}$ \\

cSVD
& $2.9532 \times 10^{\phantom{-}2}$
& $1.7337 \times 10^{-2}$
& $3.1845 \times 10^{\phantom{-}2}$
& $1.8695 \times 10^{-2}$ \\

NLP
& $2.2485 \times 10^{\phantom{-}2}$
& $1.3200 \times 10^{-2}$
& $3.6634 \times 10^{\phantom{-}2}$
& $2.1507 \times 10^{-2}$ \\

POD
& $1.1724 \times 10^{\phantom{-}2}$
& $6.8824 \times 10^{-3}$
& $4.6665 \times 10^{\phantom{-}3}$
& $2.7395 \times 10^{-1}$ \\
\bottomrule
\end{tabular}
}
\end{table}

\begin{figure}[htbp]
    \centering
    \includegraphics[width=0.8\linewidth]{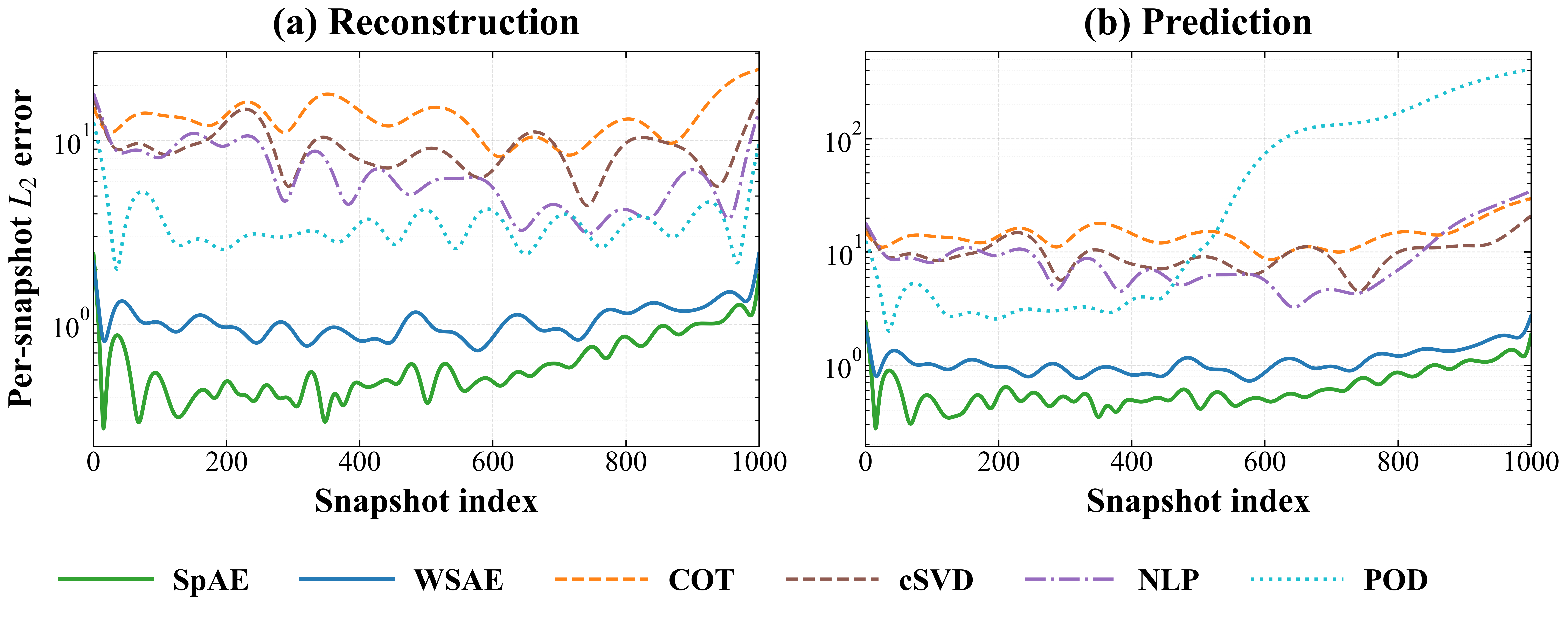}
    \caption{Per-snapshot $L_2$ errors on a logarithmic scale for the blind two-stream trajectory. Panel (a) shows the reconstruction error, and panel (b) shows the full-prediction error.}
    \label{fig:Error_twostream}
\end{figure}

We simulate an ensemble of $N=1000$ particles for the two-stream
instability, yielding a phase space in $\mathbb{R}^{2000}$. The
ground-truth trajectories are generated using the second-order symplectic
velocity Verlet scheme with $\Delta t=0.01$. Each simulation performs
1000 updates and stores both endpoints, resulting in 1001 snapshots
indexed from 0 to 1000. The state is reduced to $2k=10$. Independent Gaussian perturbations with amplitude
$\epsilon=10^{-3}$ generate 49 trajectories, of which 40 are used for
training, eight for validation, and one as the blind target trajectory.
The resulting target state matrix has shape $1001\times2000$.

The reconstruction results in \Cref{tab:twostream_errors} and the reconstruction panel of \cref{fig:Error_twostream} show that SpAE provides the most accurate reduction among the tested methods. Its error is lower than those of WSAE and POD within the $10^{-3}$ range and is more than one order of magnitude lower than those of COT, cSVD, and NLP. Moreover, the SpAE curve remains the lowest over nearly the entire blind trajectory, demonstrating the accuracy of its reduced representation.

The prediction results in \Cref{tab:twostream_errors} and the prediction panel of \cref{fig:Error_twostream} show that the latent-dynamics model trained on the SpAE representation yields the lowest long-horizon prediction error. Its error is lower than that obtained with WSAE and more than one order of magnitude lower than those obtained with all linear baselines. In \cref{fig:pred_twostream}, the phase-space vortices predicted using the SpAE representation closely match the ground truth in their formation, rotation, and deformation across all six times.

\section{Conclusion}\label{sec:conclusion}
We proposed symplecticity-preserving autoencoder (SpAE) for nonlinear model reduction of Hamiltonian systems. Its foundation is a universal approximation theory for symplectic embeddings. This reduces the approximation of a structured, manifold-valued map to the unconstrained approximation of scalar potentials, and yields an architecture that is fully expressive on nonlinear symplectic manifolds, exactly symplectic for every choice of weights, and trainable by standard unconstrained optimization.
These properties carry directly into the dynamics. The latent states evolve as a genuine Hamiltonian system. Extensive experiments on lattice and particle systems demonstrate that SpAE reduces reconstruction and prediction errors compared with existing symplectic reduction methods. The comparison with POD further highlights that accurate reconstruction alone does not guarantee stable Hamiltonian prediction.

More broadly, this work shows that geometric constraints can be incorporated directly into autoencoder architectures without sacrificing nonlinear expressivity or standard unconstrained training. This perspective goes beyond the symplectic setting and suggests new directions for learning reduced representations of physical systems with intrinsic geometric structures, including Poisson, contact, volume-preserving, metriplectic, and dissipative dynamics.
It will also be important to obtain sharper quantitative bounds, understand the interaction between representation learning and long-time integration, and combine structure-preserving architectures with geometric numerical methods. One promising direction is to extend the inverse modified differential equation \cite{zhu2024implementation,zhu2024error} to latent dynamics, thereby providing a unified tool for analyzing how learned representations and numerical integration jointly affect long-time behavior. Such developments may facilitate the construction of reliable learned surrogates for large-scale scientific computing and scientific discovery, where long-term stability and physical consistency are essential.

\section*{Data availability}
All the codes for data generation, network training, and visualization are publicly available on the GitHub repository: \url{https://github.com/LiyiFeng-Pro/SpAE}.

\section*{Funding}
L.F. and R.Z. are supported by the National Natural Science Foundation of China (No. NSFC-12271025).
Y.T. and Y.X. are supported by the National Natural Science Foundation of China (Grant No. 92470119).
A.Z. is supported by the National Research Foundation, Singapore, under its AI Singapore Programme (AISG Award No.: AISG3-RP-2022-028) and the National Research Foundation, Singapore, under the NRF fellowship (Project No. NRF-NRFF13-2021-0005).

\section*{Acknowledgments}
During the preparation of this work, the authors used large language models (LLMs) to assist with the code writing and to polish the written text. The authors thoroughly reviewed and edited the content and take full responsibility for all content.

\bibliographystyle{siamplain}
\bibliography{ref}

\begin{thebibliography}{10}

\bibitem{abid2019contrastive}
{\sc A.~Abid and J.~Zou}, {\em Contrastive variational autoencoder enhances
  salient features}, arXiv preprint arXiv:1902.04601,  (2019).

\bibitem{archilla2023spectral}
{\sc J.~F. Archilla and J.~Baj{\=a}rs}, {\em Spectral properties of exact
  polarobreathers in semiclassical systems}, Axioms, 12 (2023), p.~437.

\bibitem{archilla2019pterobreathers}
{\sc J.~F. Archilla, Y.~Doi, and M.~Kimura}, {\em Pterobreathers in a model for
  a layered crystal with realistic potentials: Exact moving breathers in a
  moving frame}, Physical Review E, 100 (2019), p.~022206.

\bibitem{Arnold1989}
{\sc V.~I. Arnol'd}, {\em Mathematical Methods of Classical Mechanics}, vol.~60
  of Graduate Texts in Mathematics, Springer, New York, 2nd ed.~ed., 1989,
  \url{https://doi.org/10.1007/978-1-4757-2063-1}.

\bibitem{bajars2023dimensionality}
{\sc J.~Baj{\=a}rs}, {\em Dimensionality reduction with proper symplectic
  decomposition for learning hamiltonian dynamics}, in International Conference
  on Numerical Computations: Theory and Algorithms, Springer, 2023, pp.~3--18.

\bibitem{bajars2025structure}
{\sc J.~Baj{\=a}rs and D.~Kalv{\=a}ns}, {\em Structure-preserving
  dimensionality reduction for learning hamiltonian dynamics}, Journal of
  Computational Physics, 528 (2025), p.~113832.

\bibitem{bajars2022data}
{\sc J.~Baj{\=a}rs and F.~Kozirevs}, {\em Data-driven intrinsic localized mode
  detection and classification in one-dimensional crystal lattice model},
  Physics Letters A, 436 (2022), p.~128071.

\bibitem{bendokat2021real}
{\sc T.~Bendokat and R.~Zimmermann}, {\em The real symplectic stiefel and
  grassmann manifolds: metrics, geodesics and applications}, arXiv preprint
  arXiv:2108.12447,  (2021).

\bibitem{blickhan2024registration}
{\sc T.~Blickhan}, {\em A registration method for reduced basis problems using
  linear optimal transport}, SIAM Journal on Scientific Computing, 46 (2024),
  pp.~A3177--A3204.

\bibitem{brantner2023symplectic}
{\sc B.~Brantner and M.~Kraus}, {\em Symplectic autoencoders for model
  reduction of hamiltonian systems}, arXiv preprint arXiv:2312.10004,  (2023).

\bibitem{braun1998nonlinear}
{\sc O.~M. Braun and Y.~S. Kivshar}, {\em Nonlinear dynamics of the
  frenkel--kontorova model}, Physics Reports, 306 (1998), pp.~1--108.

\bibitem{buchfink2023symplectic}
{\sc P.~Buchfink, S.~Glas, and B.~Haasdonk}, {\em Symplectic model reduction of
  hamiltonian systems on nonlinear manifolds and approximation with weakly
  symplectic autoencoder}, SIAM Journal on Scientific Computing, 45 (2023),
  pp.~A289--A311.

\bibitem{channell1990symplectic}
{\sc P.~J. Channell and C.~Scovel}, {\em Symplectic integration of hamiltonian
  systems}, Nonlinearity, 3 (1990), pp.~231--259.

\bibitem{chen2024constructing}
{\sc X.~Chen, B.~W. Soh, Z.-E. Ooi, E.~Vissol-Gaudin, H.~Yu, K.~S. Novoselov,
  K.~Hippalgaonkar, and Q.~Li}, {\em Constructing custom thermodynamics using
  deep learning}, Nature Computational Science, 4 (2024), pp.~66--85.

\bibitem{chen2025reduced}
{\sc Y.~Chen, W.~Guo, Q.~Tang, and X.~Zhong}, {\em Reduced-order modeling of
  hamiltonian dynamics based on symplectic neural networks}, arXiv preprint
  arXiv:2508.11911,  (2025).

\bibitem{cheng1976integration}
{\sc C.-Z. Cheng and G.~Knorr}, {\em The integration of the vlasov equation in
  configuration space}, Journal of Computational Physics, 22 (1976),
  pp.~330--351.

\bibitem{cristofaro2018symplectic}
{\sc D.~Cristofaro-Gardiner and R.~Hind}, {\em Symplectic embeddings of
  products}, Commentarii Mathematici Helvetici, 93 (2018), pp.~1--32.

\bibitem{cybenko1989approximation}
{\sc G.~Cybenko}, {\em Approximation by superpositions of a sigmoidal
  function}, Mathematics of control, signals and systems, 2 (1989),
  pp.~303--314.

\bibitem{da2008lectures}
{\sc A.~C. Da~Silva}, {\em Lectures on Symplectic Geometry}, vol.~3575,
  Springer, Berlin, 2008.

\bibitem{de1980simulation}
{\sc S.~W. de~Leeuw, J.~W. Perram, and E.~R. Smith}, {\em Simulation of
  electrostatic systems in periodic boundary conditions. i. lattice sums and
  dielectric constants}, Proceedings of the Royal Society of London. Series A,
  Mathematical and Physical Sciences,  (1980), pp.~27--56.

\bibitem{dou2010breathers}
{\sc Q.~Dou, J.~Cuevas, J.~C. Eilbeck, and F.~M. Russell}, {\em Breathers and
  kinks in a simulated crystal experiment}, arXiv preprint arXiv:1008.0278,
  (2010).

\bibitem{escande2016n}
{\sc D.~Escande, F.~Doveil, and Y.~Elskens}, {\em N-body description of debye
  shielding and landau damping}, Plasma Physics and Controlled Fusion, 58
  (2016), p.~014040.

\bibitem{feng1985proceedings}
{\sc K.~Feng}, {\em Proceedings of the 1984 Beijing Symposium on Differential
  Geometry and Differential Equations: Computation of Partial Differential
  Equations}, Science Press, 1985.

\bibitem{feng1986difference}
{\sc K.~Feng}, {\em Difference schemes for hamiltonian formalism and symplectic
  geometry}, Journal of Computational mathematics, 4 (1986), pp.~279--289.

\bibitem{feng1995collected}
{\sc K.~Feng}, {\em Collected Works of Feng Kang: II}, 1995.

\bibitem{feng2010symplectic}
{\sc K.~Feng and M.~Qin}, {\em Symplectic geometric algorithms for Hamiltonian
  systems}, vol.~449, Springer, 2010.

\bibitem{forest1990fourth}
{\sc E.~Forest and R.~D. Ruth}, {\em Fourth-order symplectic integration},
  Physica D: Nonlinear Phenomena, 43 (1990), pp.~105--117.

\bibitem{greydanus2019hamiltonian}
{\sc S.~Greydanus, M.~Dzamba, and J.~Yosinski}, {\em Hamiltonian neural
  networks}, Advances in neural information processing systems, 32 (2019).

\bibitem{hadsell2006dimensionality}
{\sc R.~Hadsell, S.~Chopra, and Y.~LeCun}, {\em Dimensionality reduction by
  learning an invariant mapping}, in 2006 IEEE computer society conference on
  computer vision and pattern recognition (CVPR'06), vol.~2, IEEE, 2006,
  pp.~1735--1742.

\bibitem{hairer2006geometric}
{\sc E.~Hairer, C.~Lubich, and G.~Wanner}, {\em Geometric numerical
  integration: structure-preserving algorithms for ordinary differential
  equations}, vol.~31, Springer Science \& Business Media, 2006.

\bibitem{hawke2024contrastive}
{\sc S.~Hawke, Y.~Ma, and D.~Li}, {\em Contrastive dimension reduction: when
  and how?}, Advances in Neural Information Processing Systems, 37 (2024),
  pp.~74034--74057.

\bibitem{hinton2006reducing}
{\sc G.~E. Hinton and R.~R. Salakhutdinov}, {\em Reducing the dimensionality of
  data with neural networks}, science, 313 (2006), pp.~504--507.

\bibitem{hornik1990universal}
{\sc K.~Hornik, M.~Stinchcombe, and H.~White}, {\em Universal approximation of
  an unknown mapping and its derivatives using multilayer feedforward
  networks}, Neural Networks, 3 (1990), pp.~551 -- 560.

\bibitem{husemoller1966fibre}
{\sc D.~Husem{\"o}ller}, {\em Fibre Bundles}, McGraw-Hill Series in Higher
  Mathematics, McGraw-Hill, New York, 1966.

\bibitem{jin2022optimal}
{\sc P.~Jin, Z.~Lin, and B.~Xiao}, {\em Optimal unit triangular factorization
  of symplectic matrices}, Linear Algebra and its Applications, 650 (2022),
  pp.~236--247.

\bibitem{jin2020unit}
{\sc P.~Jin, Y.~Tang, and A.~Zhu}, {\em Unit triangular factorization of the
  matrix symplectic group}, SIAM Journal on Matrix Analysis and Applications,
  41 (2020), pp.~1630--1650.

\bibitem{Jin2020SympNets}
{\sc P.~Jin, Z.~Zhang, A.~Zhu, Y.~Tang, and G.~E. Karniadakis}, {\em Sympnets:
  Intrinsic structure-preserving symplectic networks for identifying
  hamiltonian systems}, Neural Networks, 132 (2020), pp.~166--179.

\bibitem{kunisch2002galerkin}
{\sc K.~Kunisch and S.~Volkwein}, {\em Galerkin proper orthogonal decomposition
  methods for a general equation in fluid dynamics}, SIAM Journal on Numerical
  analysis, 40 (2002), pp.~492--515.

\bibitem{lee2020model}
{\sc K.~Lee and K.~T. Carlberg}, {\em Model reduction of dynamical systems on
  nonlinear manifolds using deep convolutional autoencoders}, Journal of
  Computational Physics, 404 (2020), p.~108973.

\bibitem{libermann1987symplectic}
{\sc P.~Libermann and C.-M. Marle}, {\em Symplectic Geometry and Analytical
  Mechanics}, vol.~35 of Mathematics and Its Applications, D. Reidel Publishing
  Company, Dordrecht, 1987.
\newblock Translated from the French by Bertram Eugene Schwarzbach.

\bibitem{maaten2008visualizing}
{\sc L.~v.~d. Maaten and G.~Hinton}, {\em Visualizing data using t-sne},
  Journal of machine learning research, 9 (2008), pp.~2579--2605.

\bibitem{mcinnes2018umap}
{\sc L.~McInnes, J.~Healy, and J.~Melville}, {\em Umap: Uniform manifold
  approximation and projection for dimension reduction}, arXiv preprint
  arXiv:1802.03426,  (2018).

\bibitem{peng2016structure}
{\sc L.~Peng and K.~Mohseni}, {\em Structure-preserving model reduction of
  forced hamiltonian systems}, arXiv preprint arXiv:1603.03514,  (2016).

\bibitem{peng2016symplectic}
{\sc L.~Peng and K.~Mohseni}, {\em Symplectic model reduction of hamiltonian
  systems}, SIAM Journal on Scientific Computing, 38 (2016), pp.~A1--A27.

\bibitem{preechakul2022diffusion}
{\sc K.~Preechakul, N.~Chatthee, S.~Wizadwongsa, and S.~Suwajanakorn}, {\em
  Diffusion autoencoders: Toward a meaningful and decodable representation}, in
  Proceedings of the IEEE/CVF conference on computer vision and pattern
  recognition, 2022, pp.~10619--10629.

\bibitem{reich1999backward}
{\sc S.~Reich}, {\em Backward error analysis for numerical integrators}, SIAM
  Journal on Numerical Analysis, 36 (1999), pp.~1549--1570.

\bibitem{sanz1988runge}
{\sc J.~M. Sanz-Serna}, {\em Runge-kutta schemes for hamiltonian systems}, BIT
  Numerical Mathematics, 28 (1988), pp.~877--883.

\bibitem{sanz2018numerical}
{\sc J.-M. Sanz-Serna and M.-P. Calvo}, {\em Numerical hamiltonian problems},
  vol.~7, Courier Dover Publications, 2018.

\bibitem{sharma2023symplectic}
{\sc H.~Sharma, H.~Mu, P.~Buchfink, R.~Geelen, S.~Glas, and B.~Kramer}, {\em
  Symplectic model reduction of hamiltonian systems using data-driven quadratic
  manifolds}, Computer Methods in Applied Mechanics and Engineering, 417
  (2023), p.~116402.

\bibitem{shen2021neural}
{\sc Z.~Shen, H.~Yang, and S.~Zhang}, {\em Neural network approximation: Three
  hidden layers are enough}, Neural Networks, 141 (2021), pp.~160--173.

\bibitem{tao2016explicit}
{\sc M.~Tao}, {\em Explicit high-order symplectic integrators for charged
  particles in general electromagnetic fields}, Journal of Computational
  Physics, 327 (2016), pp.~245--251.

\bibitem{turaev2003polynomial}
{\sc D.~Turaev}, {\em Polynomial approximations of symplectic dynamics and
  richness of chaos in non-hyperbolic area-preserving maps}, Nonlinearity, 16
  (2003), pp.~123--135.

\bibitem{vincent2008extracting}
{\sc P.~Vincent, H.~Larochelle, Y.~Bengio, and P.-A. Manzagol}, {\em Extracting
  and composing robust features with denoising autoencoders}, in Proceedings of
  the 25th international conference on Machine learning, 2008, pp.~1096--1103.

\bibitem{wang2016auto}
{\sc Y.~Wang, H.~Yao, and S.~Zhao}, {\em Auto-encoder based dimensionality
  reduction}, Neurocomputing, 184 (2016), pp.~232--242.

\bibitem{wu2024factorized}
{\sc A.~Wu and W.-S. Zheng}, {\em Factorized diffusion autoencoder for
  unsupervised disentangled representation learning}, in Proceedings of the
  AAAI Conference on Artificial Intelligence, vol.~38, 2024, pp.~5930--5939.

\bibitem{yildiz2025symplectic}
{\sc S.~Yildiz, K.~Janik, and P.~Benner}, {\em Symplectic convolutional neural
  networks}, arXiv preprint arXiv:2508.19842,  (2025).

\bibitem{zhang2016explicit}
{\sc R.~Zhang, H.~Qin, Y.~Tang, J.~Liu, Y.~He, and J.~Xiao}, {\em Explicit
  symplectic algorithms based on generating functions for charged particle
  dynamics}, Physical Review E, 94 (2016), p.~013205.

\bibitem{zhu2024implementation}
{\sc A.~Zhu, T.~Bertalan, B.~Zhu, Y.~Tang, and I.~G. Kevrekidis}, {\em
  Implementation and (inverse modified) error analysis for implicitly templated
  ode-nets}, SIAM Journal on Applied Dynamical Systems, 23 (2024),
  pp.~2643--2669.

\bibitem{zhu2022approximation}
{\sc A.~Zhu, P.~Jin, and Y.~Tang}, {\em Approximation capabilities of
  measure-preserving neural networks}, Neural Networks, 147 (2022), pp.~72--80.

\bibitem{zhu2025continuity}
{\sc A.~Zhu, Y.~Pan, and Q.~Li}, {\em Continuity-preserving convolutional
  autoencoders for learning continuous latent dynamical models from images}, in
  The Thirteenth International Conference on Learning Representations, 2025.

\bibitem{zhu2025identifiable}
{\sc A.~Zhu, B.~W. Soh, G.~A. Pavliotis, and Q.~Li}, {\em Identifiable learning
  of dissipative dynamics}, Nature Communications,  (2026).

\bibitem{zhu2024error}
{\sc A.~Zhu, S.~Wu, and Y.~Tang}, {\em Error analysis based on inverse modified
  differential equations for discovery of dynamics using linear multistep
  methods and deep learning}, SIAM Journal on Numerical Analysis, 62 (2024),
  pp.~2087--2120.

\end{thebibliography}

\end{document}